\documentclass[preprint,12pt]{elsarticle}

\usepackage[table]{xcolor}

\usepackage{amssymb}

\usepackage{lineno}

\usepackage{amsmath}
\usepackage{amssymb}
\usepackage{mathtools}
\usepackage{amsthm}

\usepackage{mymacros}

\usepackage{enumitem}
\usepackage{hyperref}

\usepackage{graphicx}
\usepackage{subfig}

\usepackage{pgfplots}
\pgfplotsset{compat=newest}
\usepgfplotslibrary{groupplots}

\usepackage{algorithm}
\usepackage{algorithmic}
\usepackage{circuitikz}

\usepackage{natbib}

\newcommand{\algrested}{\texttt{R-ed-UCB}\@\xspace}

\newcommand{\red}{\texttt{R-ed}\@\xspace}

\definecolor{color12}{RGB}{0, 153, 136}
\definecolor{color11}{RGB}{238, 119, 51}
\definecolor{black}{rgb}{0, 0, 0}

\newcommand{\hl}[1]{{\color{black}#1}}

\usepackage{multirow}

\def\appendixname{}

\allowdisplaybreaks[4]

\journal{Artificial Intelligence Journal}

\begin{document}

\begin{frontmatter}



\title{Rising Rested Bandits: \\ Lower Bounds and Efficient Algorithms}


\author[inst1]{Marco Fiandri}
\author[inst1]{Alberto Maria Metelli}
\author[inst1]{Francesco Trovò}

\affiliation[inst1]{organization={Dipartimento di Elettronica, Informazione e Bioingegneria, Politecnico di Milano},
            addressline={Piazza Leonardo da Vinci, 32}, 
             postcode={20133}, 
            city={Milano},
            country={Italy}}

\begin{abstract}{
This paper is in the field of stochastic Multi-Armed Bandits (MABs), \ie those sequential selection techniques able to learn online using 
only the feedback given by the chosen option (a.k.a.~arm). We study a particular case of the rested bandits in which the arms' 
expected reward is monotonically non-decreasing and concave. We study the inherent sample complexity of the regret minimization
problem by deriving suitable regret lower bounds. Then, we design an algorithm for the rested case (\algrested), providing a regret bound depending on the properties of 
the instance and, under certain circumstances, of $\widetilde{\mathcal{O}}(T^{\frac{2}{3}})$. We empirically compare our algorithms with 
state-of-the-art methods for non-stationary MABs over several synthetically generated tasks and an online model selection problem for a 
real-world dataset.}
\end{abstract}




\end{frontmatter}

\section{Introduction}
The classical stochastic MAB framework~\citep{lattimore2020bandit} has been successfully applied to a number of applications, such as advertising, recommendation, and networking.
MABs model the scenario in which a learner sequentially selects (a.k.a.~pulls) an option (a.k.a.~arm) in a finite set and receives feedback (a.k.a.~reward) corresponding to the chosen option.
The goal of \emph{regret-minimization} algorithms is to guarantee the \emph{no-regret} property, meaning that the loss due to not knowing the best arm is increasing sublinearly with the learning horizon. 
One of the assumptions that allows designing no-regret algorithms consists in requiring that the expected reward provided by the available options is \emph{stationary}, \ie rewards come from a fixed distribution throughout the learning process.

However, the arms' expected reward may change over time due to intrinsic modifications of the arms or the environment.
The \emph{adversarial} algorithms offer a no-regret approach, in which no assumption on the nature of the reward is required.
It has been shown that, in this setting, it is possible to design effective algorithms,  e.g., \texttt{EXP3}~\citep{auer1995gambling}. However, in practice, their performance is unsatisfactory for several reasons. First, the \emph{non-stationarity} of real-world cases is far from being adversarial. Second, the guarantees delivered by such algorithms focus on a (static) regret notion, competing against a constant (best) action, are often poorly informative in real-world scenarios~\citep{DekelTA12}.
Instead, non-stationarity is explicitly accounted for by a surge of methods that consider either abrupt changes~\citep[\eg][]{garivier2011upper}, smoothly changing environments~\citep[\eg][]{fiandri2024sliding}, or bounded reward variation~\citep[\eg][]{besbes2014stochastic}. Under these regularity assumptions on the non-stationary evolution, it is possible to design provably efficient algorithms minimizing a stronger notion of (dynamic) regret, where the algorithm is evaluated against a sequence of (best) actions.


While in non-stationary MABs, the arms' expected reward changes as an effect of \emph{nature}, a different setting arises when the expected reward changes as an effect of \emph{pulling} the arm. This is the case of \emph{rested} bandits~\citep{tekin2012online}. It is worth stressing the distinction between rested and non-stationary bandits (a.k.a.~\emph{restless}). Indeed, in the former case, the evolution of the arm is triggered by the algorithm itself when pulling it. This makes regret minimization in rested bandits more challenging compared to the non-stationary one. Indeed, regularity assumptions on the arms' expected reward evolution are needed~\cite{heidari2016tight, seznec2019rotting}. The literature distinguishes between \emph{rotting} and \emph{rising} rested bandits. In rotting rested bandits~\citep{seznec2019rotting}, the expected reward of the arms is monotonically non-increasing as a function of the pulls, modeling degradation phenomena. Knowing the monotonicity property allows deriving more specialized algorithms, exploiting the process characteristics. Notably, the symmetric problem, in which monotonically non-decreasing expected rewards are enforced, cannot be addressed with the same approaches. Indeed, it was shown that it represents a significantly more complex problem, even for deterministic arms~\citep{heidari2016tight}. In this non-decreasing setting, a common assumption is the concavity of the expected reward function that defines the \emph{rising} rested bandits setting~\citep{li2020efficient}.

This paper aims to study the regret minimization problem in the setting of the \emph{stochastic} rested bandits when the arms' expected reward is monotonically non-decreasing. This setting arises in several real-world sequential decision-making problems. For instance, suppose we have to choose among a set of optimization algorithms to maximize an unknown stochastic concave function. In this setting, we expect that all the algorithms progressively \emph{increase} (on average) the function value and eventually converge to the optimum, possibly with different speeds. Therefore, we wonder which candidate algorithm to assign the available resources (\eg computational power or samples) to identify the one that converges faster on average to the optimum. This \emph{online model selection} process can be represented as a \emph{rested} MAB~\citep{tekin2012online} with non-decreasing expected rewards. Indeed, each optimization algorithm (arms) and the function value do not evolve if we do not select (pull) it. Other examples of such settings are represented by the so-called Combined Algorithm Selection and Hyperparameter optimization (CASH,~\citet{thornton2013auto},~\citet{li2020efficient}), whose goal is to identify the best learning algorithm together with the best hyperparameter configuration for a given machine learning task, one of the most fundamental problems in Automatic Machine Learning (AutoML). The setting of the concave rising bandit is also extremely useful when modeling satiation effects in recommendations (\citet{clerici2023linear},\citet{xu2023online}), and concave learning curves have been shown to arise in
various laboratory environments (\eg see ~\citet{jovanovic1995bayesian}; ~\citet{anderson1991reflections}) and is very natural and common in the context of human
learning (\eg see the work by~\citet{SonMeta}).

\paragraph{Original Contribution}
In this paper, we address the regret minimization problem in stochastic rising bandits, i.e., stochastic bandits, in which the expected rewards are monotonically non-decreasing and concave as a function of the number of pulls. More specifically:

\begin{itemize}
	\item we provide a set of regret lower bounds highlighting the inherent complexity of regret minimization in rising rested bandits. In particular, we show that the non-decreasing and concavity properties are not enough to guarantee the existence of no-regret algorithms;
	\item we design \algrested and optimistic algorithm for the rising rested bandits that make use of a suitably 
	designed estimator for the expected reward;
	\item we show that \algrested suffer an expected regret that depends on the expected reward function profile and, under some conditions, of order $\widetilde{\mathcal{O}}(T^{\frac{2}{3}})$;\footnote{With $\widetilde{\mathcal{O}}(\cdot)$ we disregard logarithmic terms in the order.}
	\item we illustrate, using synthetic and real-world
	data, the effectiveness of our approaches, compared with state-of-the-art algorithms for the non-stationary bandits.
\end{itemize}

\section{Related Works}
In this section, we survey the literature connected to the settings and approaches presented in this paper.

\paragraph{Restless and Rested Bandits}
The \emph{rested} and \emph{restless} bandit settings have been introduced by~\citet{tekin2012online} and further developed by~\citet{ortner2012regret,russac2019weighted} in the restless version and by~\citet{mintz2020nonstationary,pike2019recovering} in the rested one. Originally the evolution of the expected reward was modeled via a stochastic process, \eg a Markov chain with finite state space or a linear regression process. For instance,~\citet{NEURIPS2020_89ae0fe2} proposes an optimistic approach based on the estimation of the transition kernel of the underlying chain. More recently, the terms rested and restless have been employed to denote arms whose expected reward changes as time passes, for restless ones, or whenever being pulled, for rested ones~\citep{seznec2019rotting,seznec2020asingle}. That is the setting we target in this work.

\paragraph{Non-Stationary Bandits}
The restless bandits, without a fixed temporal reward evolution, are usually addressed via non-stationary MAB approaches, that include both passive~\citep[\eg][]{garivier2011upper,besbes2014stochastic,auer2019adaptively,fiandri2024sliding}
and active~\citep[\eg][]{liu2018change,besson2019efficient,cao2019nearly} methods. The former algorithms base their selection criterion on the most recent feedbacks, while the latter actively try to detect if a change in the arms' rewards occurred and use only data gathered after the last change. \citet{garivier2011upper} employ a discounted reward approach (\texttt{D-UCB}) or an adaptive sliding window (\texttt{SW-UCB}), proving a $\widetilde{\mathcal{O}}(\sqrt{T})$ regret when the number of abrupt changes is known. Similar results have been obtained by~\citet{auer2019adaptively} without knowing the number of changes, at the price of resorting to the doubling trick. \citet{besbes2014stochastic} provide an algorithm, namely \texttt{RExp3}, a modification \texttt{EXP3}, originally designed for adversarial MABs, to give a regret bound of $\mathcal{O}(T^{\frac{2}{3}})$ under the assumption that the total variation $V_T$ of the arms' expected reward is known. The knowledge of $V_T$ has been removed by~\citet{chen2019anew} using the doubling trick. In~\citet{fiandri2024sliding}, an approach in which the combined use of a sliding window on a Thompson Sampling-like algorithm provides theoretical guarantees both on abruptly and smoothly changing environments. Nonetheless, in our setting, their result might lead to linear regret for specific instances. Notably, none of the above explicitly use assumptions on the monotonicity of the expected reward over time.

\paragraph{Rising Bandits}
The \emph{rising} bandit problem has been tackled in its deterministic version by~\citet{heidari2016tight,li2020efficient}. In~\citet{heidari2016tight}, the authors design an online algorithm to minimize the regret of selecting an increasing and concave function among a finite set. This study assumes that the learner receives feedback about the true value of the reward function, \ie no stochasticity is present. In~\citet{li2020efficient}, the authors model the problem of parameter optimization for machine learning models as a rising bandit setting. They propose an online algorithm having good empirical performance, still in the case of deterministic rewards. A case where the reward is increasing in expectation (or equivalently decreasing in loss), but no longer deterministic, is provided by~\citet{cella2021best}.
However, the expected reward follows a given parametric form known to the learner, who estimates such parameters in the best-arm identification and regret-minimization frameworks. The need for knowing the parametric form of the expected reward makes these approaches hardly applicable for arbitrary increasing functions. Recently, a surge of approaches has been designed for addressing other learning problems in stochastic rising bandits, including best-arm identification~\citep{TakemoriUG24,mussibest}. Finally, \citet{genalti2024graph} proposes a novel framework that interpolates between rested and restless bandits, still assuming the rising condition.

\paragraph{Corralling Bandits}
It is also worth mentioning the \emph{corralling} bandits~\citep{agarwal2017corralling, pacchiano2020model, abbasi2020regret, pacchiano2020regret, arora2021corralling}, a setting in which the goal is to minimize the regret of a process choosing among a finite set of bandit algorithms. This setting, close to online model selection, is characterized by particular assumptions. Indeed, each arm corresponds to a learning algorithm, operating on a bandit, endowed with a (possibly known) regret bound, sometimes requiring additional conditions (\eg stability).

\section{Problem Setting}\label{sec:problemSetting}
A rested $K$-armed MAB\footnote{We refer to the definition of~\citep{levine2017rotting,seznec2020asingle} and not to the one of~\citep{tekin2012online} that assumes an underlying Markov chain governing the arms' distributions.} is defined as a vector of probability distributions $\bm{\nu} = (\nu_i)_{i \in [K]}$, where $\nu_i : \Nat \rightarrow \Delta(\Reals)$ depends on a parameter $n \in \Nat$ for every $i \in [K]$, where $[K] \coloneq \{1, \ldots, K\}$. Let $T \in \Nat$ be the optimization horizon, at each round $t \in [T]$, the agent selects an arm $I_t \in [K]$ and observes a reward $R_{t} \sim \nu_{I_{t}}(N_{I_t,t})$, where $N_{i,t} = \sum_{l=1}^{t} \Ind\{I_l=i\}$ is the number of times the arm $i \in [K]$ was pulled up to round $t$ so that the reward depends on the number of pulls $N_{I_t,t} = N_{I_t,t-1}+1$ of arm $I_t$ up to $t$. Thus, the expected reward of a rested arm changes when being pulled and, therefore, it models phenomena that evolve as a consequence of the agent intervention.

For every arm $i \in [K]$, we define its expected reward $\mu_i : \Nat \rightarrow \Reals$ as the expectation of the reward, \ie $\mu_i(n) = \E_{R \sim \nu_{i}(n)}[R]$ and denote the vector of expected rewards as $\bm{\mu} = (\mu_i)_{i \in [K]}$. We assume that the expected rewards are bounded in $[0,1]$, and that the rewards are $\sigma^2$-subgaussian, \ie $\E_{R \sim \nu_{i}(n)}[e^{\lambda  (R-\mu_i(t,n))}] \le e^{\frac{\sigma\lambda^2}{2}}$, for every $\lambda \in \Reals$.

%

\paragraph{Rising Bandits}

We revise the \emph{rising} bandits notion, \ie MABs with expected rewards \emph{non-decreasing} and \emph{concave} as a function of $n$~\citep{heidari2016tight}.\footnote{Deterministic bandits with non-decreasing expected rewards were introduced in~\citep{heidari2016tight} with the term \emph{improving}. In~\citep{li2020efficient}, the term \emph{rising} was used to denote the improving bandits with concave expected rewards (concavity was already employed by~\citet{heidari2016tight}).}

\begin{ass}[Non-Decreasing expected reward]\label{ass:incr}
Let $\bm{\nu}$ be a MAB, for every arm $i\in [K]$, function $\mu_i(\cdot)$ is non-decreasing. In particular, we define the \emph{increments} $\gamma_i(n) \coloneqq \mu_i(n+1) - \mu_i(n) \ge 0$ for every $n \in \Nat$.\footnote{For $n = 0$, we conventionally set $\gamma_i(0) = \mu_i(1)$.}
\end{ass}
From an economic perspective, $\gamma_i(\cdot)$ represents the \emph{increase of total return} (or expected reward) we obtain by adding a factor of production, \ie pulling the arm.
In the next sections, we analyze how the following assumption defines a remarkable class of bandits with non-decreasing expected rewards~\citep{heidari2016tight}.

\begin{ass}[Concave expected reward]\label{ass:decrDeriv}
Let $\bm{\nu}$ be a MAB, for every arm $i\in [K]$, function $\mu_i(\cdot)$ is concave, \ie $\gamma_i(n+1) - \gamma_i(n) \le 0$.
\end{ass}

As pointed out by~\citet{heidari2016tight}, the concavity assumption corresponds, in economics, to the \emph{decrease of marginal returns} that emerges when adding a factor of production, \ie pulling the arm (rested) or letting time evolve for one unit (restless).

Formally, we define \emph{rising} a stochastic MAB in which Assumption~\ref{ass:incr} and Assumption~\ref{ass:decrDeriv} hold.

\paragraph{Learning Problem}
Let $t \in [T]$ be a round, we denote with $\Hs_t = (I_l,R_l)_{l=1}^t$ the \emph{history} of observations up to $t$. A (non-stationary) deterministic policy is a function $\pi : \Hs_{t-1} \mapsto I_t$ mapping a history to an arm. For the sake of concision, we will write $\pi(t) \coloneqq \pi(\Hs_{t-1})$. The performance of a policy $\pi$ in a MAB with expected rewards $\bm{\mu}$ is the \emph{expected cumulative reward} collected over the $T$ rounds, formally:
\begin{align*}
    J_{\bm{\mu}}({\pi},T) \coloneq  \mathbb{E}_{\bm{\mu},\pi} \left[\sum_{t \in [T]} \mu_{I_t} \left(N_{I_t,t}\right) \right],
\end{align*}
and the expectation is computed over the histories.
A policy ${\pi}^*_{\bm{\mu},T}$ is \emph{optimal} if it maximizes the expected cumulative reward: ${\pi}^*_{\bm{\mu},T} \in \argmax_{\pi} \{ J_{\bm{\mu}}({\pi},T) \}$. Denoting with $J_{\bm{\mu}}^*\left(T\right) \coloneq J_{\bm{\mu}}({\pi}^*_{\bm{\mu},T},T)$ the expected cumulative reward of an optimal policy, the suboptimal policies $\pi$ are evaluated via the \emph{expected cumulative regret} $R_{\bm{\mu}}(\pi,T)$, formally defined as follows:
\begin{align}\label{eq:regret}
	R_{\bm{\mu}}(\pi,T) \coloneq J_{\bm{\mu}}^*\left(T\right) - J_{\bm{\mu}}\left(\pi,T\right).
\end{align}

\paragraph{Problem Characterization}
To characterize the problem instance, we introduce the following quantity, namely the \emph{cumulative increment}, defined for every $M \in [T]$ and $q \in [0,1]$ as:
\begin{align}\label{eq:magicQuantity}
	& \Upsilon_{\bm{\mu}}(M,q) \coloneq  \sum_{l=1}^{M-1} \max_{i \in [K]}\{\gamma_i(l)^q\}.
\end{align}
The cumulative increment accounts for how fast the arms expected rewards reach their asymptotic value, \ie become stationary. Intuitively, small values of $\Upsilon_{\bm{\mu}}(M,q)$ lead to simpler problems, as they are closer to stationary bandits. It is worth noting that $0 \le \Upsilon_{\bm{\mu}}(M,q) \le K M^{1-q}$. Table~\ref{tab:rates} reports some bounds on $\Upsilon_{\bm{\mu}}(M,q)$ for particular choices of $\gamma_i(l)$ and $q$.\footnote{Formal proofs for the reported bounds are provided in Lemma~\ref{lemma:lemmaBoundsGamma}.} When $q=1$, the cumulative increment $\Upsilon_{\bm{\mu}}(T,1)$ corresponds to the total variation $V_T \coloneq \sum_{l=1}^{T-1} \max_{i \in [K]} \left\{\gamma_i(l) \right\}$~\citep{besbes2014stochastic}.


\begin{table}
\caption{$\mathcal{O}(\cdot)$ rates of $\Upsilon_{\bm{\mu}}(M,q)$ in the case $\gamma_i(l) \le f(l)$ for all $i \in [K]$ and $l \in \Nat$.}\label{tab:rates}
\centering\renewcommand{\arraystretch}{1.5}
\begin{tabular}{r|C{2cm}|C{2cm}|C{2cm}|C{2cm}|}
\cline{2-5}
$f(l)$ &  $e^{-cl}$ & $l^{-c}$ ($cq > 1$) & $l^{-c}$ ($cq = 1$) & $l^{-c}$ ($cq \le 1$)\\\cline{2-5}
  \rule{0pt}{15pt} $\Upsilon_{\bm{\mu}}(M,q)$ & $\displaystyle \frac{e^{-cq}}{cq}$ & $\displaystyle \frac{1}{cq-1}$ &  $\displaystyle \log M$ & $\displaystyle \frac{M^{1-cq}}{1-cq}$ \\\cline{2-5}
\end{tabular}
\end{table}

%
%
%

\paragraph{Optimal Policy} 
We recall that the \emph{oracle constant} policy is the one that always plays at each round $t \in [T]$ the arm that maximizes the sum of the expected rewards over the horizon $T$, is optimal for the non-increasing rested bandits. Formally:

\begin{restatable}[\citet{heidari2016tight}]{thr}{thrRestedOptimal}\label{thr:thrRestedOptimal}
	Let $\pi^\star_{\bm{\mu},T}$ be the \emph{oracle constant} policy:
	\begin{align*}
		\pi^\star_{\bm{\mu},T}(t) \in \arg \max_{i \in [K]} \Bigg\{ \sum_{l \in [T]} \mu_i(l)\Bigg\}, \quad \forall t \in [T].
	\end{align*}
	Then, $\pi^\star_{\bm{\mu},T}$ is optimal for the rested non-decreasing bandits (\ie under Assumption~\ref{ass:incr}). We will denote with $\pi^\star_{\bm{\mu},T}(t) \equiv : i^\star(T)$ the optimal constant arm.
\end{restatable}
The result holds under the non-decreasing property (Assumption~\ref{ass:incr}) only, without requiring concavity (Assumption~\ref{ass:decrDeriv}). However, this policy cannot be computed in practice as it requires knowing the full function $\mu_i(\cdot)$ in advance. 

\clearpage
\section{Lower Bounds for Regret Minimization in Stochastic Rising Bandits}\label{sec:lbs}

In this section, we provide a set of regret lower bounds highlighting the inherent and statistical challenges of
the regret minimization problem in stochastic rising bandits.
At first, in Section~\ref{sec:prelimRes}, we provide
a technical result extending for our setting the classical result that rewrites the regret in terms of the number of pulls for the suboptimal arms. Then, in Section~\ref{sec:restedNon}, we
discuss some non-learnability results for deterministic rising bandits and, in Section~\ref{sec:stochrestlow}, extend these results for the stochastic case. Finally, in Section~\ref{sec:ups}, we show a lower bound that highlights the dependence on the
complexity term $\Upsilon_{\bm{\mu}}(T,q)$.

\subsection{Regret Decomposition}\label{sec:prelimRes}
To provide a meaningful construction of the lower bounds, some additional notation is needed. Let $T \in \mathbb{N}$ be the time horizon, the arm's \emph{average expected reward} for every arm $i \in [K]$ is defined as:
\begin{align}
    \overline{\mu}_i(T) \coloneqq \frac{1}{T}\sum_{n \in [T]}\mu_i(n).
\end{align}
This allows rephrasing the optimal arm as $i^{\star}(T) \in \arg \max_{i \in [K]}  \overline{\mu}_i(T)$. Furthermore, we define the average expected regret suffered when playing arm $i \in [K]$ instead of the optimal one $i^{\star}(T)$ as follows:
\begin{align}
    \overline{\Delta}_i  \coloneqq \overline{\mu}_{i^{\star}(T)}(T)-\overline{\mu}_{i}(T).
\end{align}
In the following, for a 2-armed bandit problem, we omit the subscript and denote it just as $\overline{\Delta}$. These definitions allow obtaining a technical result that will be of crucial importance in the derivation of the regret lower bounds that holds for \textit{any} rising rested bandit.

\begin{restatable}[Regret Decomposition]{lemma}{lowerBoundWithDeltaBar}\label{lemma:Regret decomposition}
Let $T \in \mathbb{N}$ be the time horizon, $\pi$ a fixed policy, and $\bm{\mu}$ a rising rested bandit (Assumptions~\ref{ass:incr} and~\ref{ass:decrDeriv}). Then, it holds that:
\begin{align}
    R_{\bm{\mu}}(\pi,T) \ge \sum_{i \neq i^{\star}(T)} \overline{\Delta}_i \mathbb{E}_{\bm{\mu},\pi}[N_{i,T}].
\end{align}
\end{restatable}

Lemma~\ref{lemma:Regret decomposition} establishes a relationship between the regret suffered in the original rising bandit and the regret suffered in a standard bandit by any policy that plays the arms the same expected number of times  $\mathbb{E}_{\bm{\mu},\pi}[N_{i, T}]$ and in which the arms have as expected rewards the values $\overline{\mu}_i(T)$. As we shall see in the next sections, this result allows for simplifying the computation of the expected regret in the lower bound construction.

\subsection{Lower Bounds for the Deterministic Setting}\label{sec:restedNon} 
We now prove a result highlighting the challenges of the non-decreasing rested bandits. We show that with no assumptions on the expected reward $\mu_i(n)$ (\eg concavity), it is impossible to devise a no-regret algorithm.

\begin{restatable}[Non-Learnability under Assumption~\ref{ass:incr}]{thr}{nonLearnable}\label{thr:nonLearnable}
	Let $\mathcal{M}_1(\gamma_{\max})$ be the set of $2$-armed non-decreasing (Assumption~\ref{ass:incr}) deterministic rested bandit with $\gamma_i(n) \le \gamma_{\max} \le 1$ for all $i \in [K]$ and $n \in \Nat$. For every learning policy $\pi$, it holds that:
	\begin{align*}
		\sup_{\bm{\mu} \in \mathcal{M}_1(\gamma_{\max})} R_{\bm{\mu}}(\pi, T) \ge \left\lfloor \frac{\gamma_{\max}}{12} T \right\rfloor. 
	\end{align*}
\end{restatable}

Some observations are due to provide an intuition of the result. To find a regret lower bound it is necessary to select (at least) two bandit problems balancing two conflicting requirements: (i) the instances should be built so that a sequence of actions that is good for one bandit is not good for the other, and (ii)  we shall ask for this instances to be sufficiently similar so that the policy cannot statistically identify the true bandit with
reasonable statistical accuracy.
As illustrated in Figure~\ref{fig:instance1}, it is possible to design two 2-armed instances ($A$ and $B$) that are \emph{identical} up to a number of pulls \emph{linearly proportional} to the time horizon $T$ in a way that the minimum error we can collect at every pull, depending on $\overline{\Delta} = \Omega(1)$ (as highlighted by Lemma~\ref{lemma:Regret decomposition}), is a non-decreasing function of $T$.

The intuition behind this result is that if we enforce no condition on the increment $\gamma_i(n)$, we cannot predict how much the arm expected reward will increase in the future. Therefore, we face the dilemma of whether or not to pull an arm that is currently believed to be suboptimal, hoping its expected reward will increase. If we decide to pull it and its expected reward will not actually increase, or if we decide not to pull it and its expected reward will actually increase, becoming optimal, we will suffer linear regret.

\begin{figure}[t]
\begin{center}
\resizebox{\textwidth}{!}{
\begin{tikzpicture}[
  declare function={
    func(\x)= (\x < 3) * (0) + (\x >= 3) * (1);
  }
]
\begin{groupplot}[
  group style={
    group size=2 by 1,
    vertical sep=0pt,
    horizontal sep=3cm,
    group name=G},
     width=7cm,height=5cm,
  axis lines=middle,
  legend style={at={(1.2,1)},anchor=north,legend cell align=left} 
  ]
  \nextgroupplot[xmin=0, xmax=13, ymax=1.1, ymin=-0.02, domain=0:12, samples=1000, xtick = {0,3,12}, xticklabels = {$0$, $\lfloor\frac{T}{3}\rfloor$, $T$}, xlabel={$n$}]
  \addplot[red,  ultra thick] (x,.5);\addlegendentry{$\mu^A_1$}
  \addplot[blue,  ultra thick] {func(x)};\addlegendentry{$\mu^A_2$}
  \addplot[blue,  ultra thick, densely dashed] (x,0.33);\addlegendentry{$\overline{\mu}^A_2$}
  \draw [ <->, >=Stealth] (1,0.333) -- (1,0.5) node[pos=0.5,right]{\footnotesize $\overline{\Delta}$};
  
  \nextgroupplot[xmin=0, xmax=13, ymax=1.1, ymin=-0.02, domain=0:12, samples=1000, xtick = {0,12}, xticklabels = {$0$, $T$},xlabel={$n$}]
  \addplot[red,  ultra thick] (x,.5);\addlegendentry{$\mu^B_1$}
  \addplot[blue, ultra thick] (x,0);\addlegendentry{$\mu^B_2$}
  \addlegendentry{$\mu^B_1$}
  \draw [ <->, >=Stealth] (1,0) -- (1,0.5) node[pos=0.5,right]{\footnotesize $\overline{\Delta}$};
\end{groupplot}
\end{tikzpicture}
}
\end{center}
\caption{The two 2-armed instances (A on the left, B on the right) of rested bandits with non-decreasing expected rewards used in the proof of Theorem~\ref{thr:nonLearnable}.}\label{fig:instance1}
\end{figure}
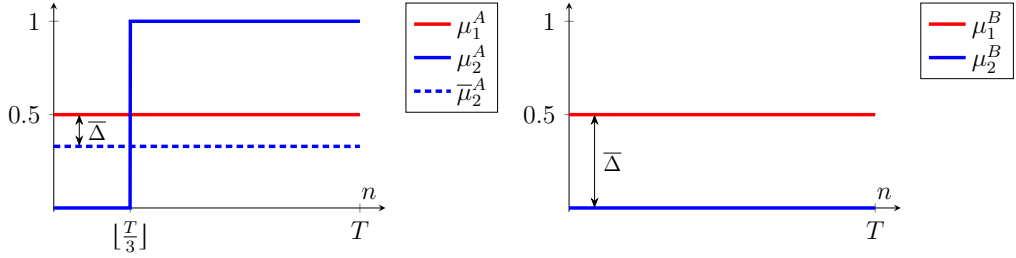

Theorem~\ref{thr:nonLearnable} highlights the importance of Assumption~\ref{ass:decrDeriv} regarding concavity. However, the following result shows that, even for concave expected rewards it is not possible to design no-regret algorithms.

\begin{restatable}[Non-Learnability under Assumptions~\ref{ass:incr} and~\ref{ass:decrDeriv}]{thr}{nonLearnableLinear}\label{thr:nonLearnable2}
	Let $\mathcal{M}_2$ be the set of $2$-armed rising (Assumptions~\ref{ass:incr} and~\ref{ass:decrDeriv}) deterministic rested bandit. For every learning policy $\pi$, it holds that:
	\begin{align*}
		\sup_{\bm{\mu} \in \mathcal{M}_2} R_{\bm{\mu}}(\pi, T) \ge \left\lfloor  \frac{T}{64} \right\rfloor. 
	\end{align*}
\end{restatable}

Theorem~\ref{thr:nonLearnable2} illustrates that there exist instances of the rising rested bandit for which the regret is at least linear. Figure~\ref{fig:instance2} illustrates how, even when concavity is present, the linearity of the expected reward function $\mu_i(\cdot)$ does not provide a sufficiently strong structure to predict the future behaviour of the arms effectively. Indeed, instances $A$ and $B$ are identical up to $\lfloor \frac{T}{2} \rfloor$ but, then, arm $2$ changes its behavior and becomes optimal for instance $B$ (while for instance $A$, the optimal arm is $1$). Intuitively, even if concavity is present, to distinguish the instances both arms have to be pulled a number of times that is \emph{linearly proportional} to $T$.

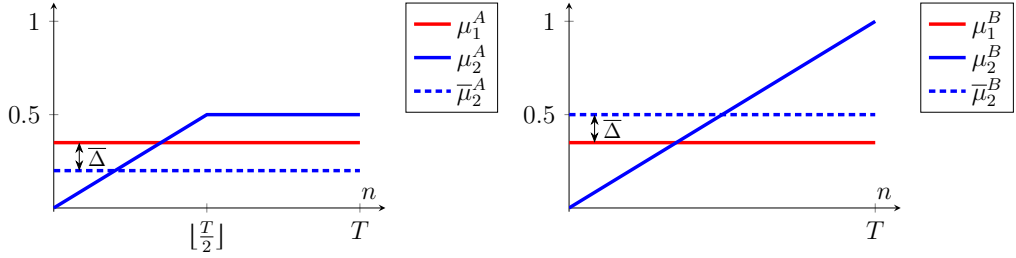
\begin{figure}[t]
\begin{center}
\resizebox{\textwidth}{!}{
\begin{tikzpicture}[
  declare function={
    func(\x)= (\x < 6) * (0.0833 * \x) + (\x >= 6) * (.5);
  }
]
\begin{groupplot}[
  group style={
    group size=2 by 1,
    vertical sep=0pt,
    horizontal sep=3cm,
    group name=G},
     width=7cm,height=5cm,
  axis lines=middle,
  legend style={at={(1.2,1)},anchor=north,legend cell align=left} 
  ]
  \nextgroupplot[xmin=0, xmax=13, ymax=1.1, ymin=-0.02, domain=0:12, samples=1000, xtick = {0,6,12}, xticklabels = {$0$, $\lfloor\frac{T}{2}\rfloor$, $T$}, xlabel={$n$}]
  \addplot[red,  ultra thick] (x,0.35);\addlegendentry{$\mu^A_1$}
  \addplot[blue,  ultra thick] {func(x)};\addlegendentry{$\mu^A_2$}
  \addplot[blue,  ultra thick, densely dashed] (x,0.2);\addlegendentry{$\overline{\mu}^A_2$}
  \draw [ <->, >=Stealth] (1,0.2) -- (1,0.35) node[pos=0.5,right]{\footnotesize $\overline{\Delta}$};
  
  \nextgroupplot[xmin=0, xmax=13, ymax=1.1, ymin=-0.02, domain=0:12, samples=1000, xtick = {0,12}, xticklabels = {$0$, $T$},xlabel={$n$}]
  \addplot[red,  ultra thick] (x,0.35);\addlegendentry{$\mu^B_1$}
  \addplot[blue, ultra thick] (x,0.0833*x);\addlegendentry{$\mu^B_2$}
  \addplot[blue,  ultra thick, densely dashed] (x,0.5);\addlegendentry{$\overline{\mu}^B_2$}
    \draw [ <->, >=Stealth] (1,0.35) -- (1,0.5) node[pos=0.5,right]{\footnotesize $\overline{\Delta}$};
\end{groupplot}
\end{tikzpicture}
}
\end{center}
\caption{The two 2-armed instances (A on the left, B on the right) of rested bandits with non-decreasing expected rewards used in the proof of Theorem~\ref{thr:nonLearnable2}.}\label{fig:instance2}
\end{figure}

The following corollary shows under which conditions the problem can be learnable, unveiling what factors contribute to the complexity of the rising rested problem.

\begin{restatable}[Learnability under Assumptions~\ref{ass:incr} and~\ref{ass:decrDeriv}]{coroll}{LearnableLinearbeta}\label{thr:Learnablebeta}
	Let  $\beta \in [0,1)$ and $\mathcal{M}_{2,\beta}$ be the set of $2$-armed rising (Assumptions~\ref{ass:incr} and~\ref{ass:decrDeriv}) deterministic rested bandit that can be identical only up to $\lfloor T^{\beta} \rfloor \le \frac{T}{2}$ pull. For every learning policy $\pi$, it holds that:
	\begin{align*}
		\sup_{\bm{\mu} \in \mathcal{M}_{2,\beta}} R_{\bm{\mu}}(\pi, T) \ge  \left\lfloor \frac{T^{\beta}}{32} \right\rfloor. 
	\end{align*}
\end{restatable}

This corollary follows the same ideas as Theorem~\ref{thr:nonLearnable2} with the only difference that the additional parameter $\beta$
controls the maximum round in which the two instances $A$ and $B$ are indistinguishable, i.e., $\lfloor T^\beta \rfloor$, as highlighted in Figure \ref{fig:instance4}. The lower bound highlights how the regret scales proportionally to $\lfloor T^\beta \rfloor$ unveiling that regret minimization
in deterministic rising bandits is possible when $\beta < 1$.

\begin{figure}[t]
	\begin{center}
		\resizebox{\textwidth}{!}{
			\begin{tikzpicture}[
				declare function={
					func(\x)= (\x < 5) * (0.0833 * \x) + (\x >= 5) * (.4165);
				}
				]
				\begin{groupplot}[
					group style={
						group size=2 by 1,
						vertical sep=0pt,
						horizontal sep=3cm,
						group name=G},
					width=7cm,height=5cm,
					axis lines=middle,
					legend style={at={(1.2,1)},anchor=north,legend cell align=left} 
					]
					\nextgroupplot[xmin=0, xmax=13, ymax=1.1, ymin=-0.02, domain=0:12, samples=1000, xtick = {0,5,12}, xticklabels = {$0$, $\lfloor T^\beta \rfloor$, $T$}, xlabel={$n$}]
					\addplot[red,  ultra thick] (x,0.33);\addlegendentry{$\mu^A_1$}
					\addplot[blue,  ultra thick] {func(x)};\addlegendentry{$\mu^A_2$}
					\addplot[blue,  ultra thick, densely dashed] (x,0.175);\addlegendentry{$\overline{\mu}^A_2$}
					\draw [ <->, >=Stealth] (1,0.175) -- (1,0.33) node[pos=0.5,right]{\footnotesize $\overline{\Delta}$};
					
					\nextgroupplot[xmin=0, xmax=13, ymax=1.1, ymin=-0.02, domain=0:12, samples=1000, xtick = {0,12}, xticklabels = {$0$, $T$},xlabel={$n$}]
					\addplot[red,  ultra thick] (x,0.33);\addlegendentry{$\mu^B_1$}
					\addplot[blue, ultra thick] (x,0.0833*x);\addlegendentry{$\mu^B_2$}
					\addplot[blue,  ultra thick, densely dashed] (x,0.5);\addlegendentry{$\overline{\mu}^B_2$}
					\draw [ <->, >=Stealth] (1,0.33) -- (1,0.5) node[pos=0.5,right]{\footnotesize $\overline{\Delta}$};
				\end{groupplot}
			\end{tikzpicture}
		}
	\end{center}
	\caption{The two 2-armed instances (A on the left, B on the right) of rested bandits with non-decreasing expected rewards used in the proof of Corollary~\ref{thr:Learnablebeta}.}\label{fig:instance4}
\end{figure}
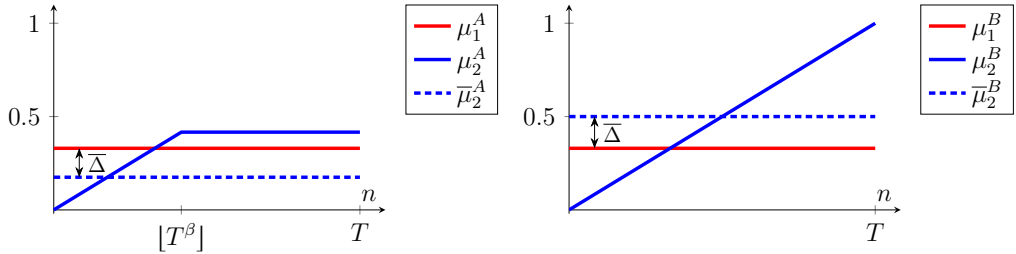

\subsection{Lower Bounds for the Stochastic Setting}\label{sec:stochrestlow}
If we move to the analysis of the problem in the stochastic setting, we have the following general minimax lower bound.
\begin{restatable}[Lower Bound under Assumptions~\ref{ass:incr} and~\ref{ass:decrDeriv}]{thr}{apprendiamopoco}\label{thr:apprendiamopoco}
	Let $\beta \in [0,1)$ and let $\mathcal{N}_{2,\beta}$ be the set of $2$-armed rising (Assumptions~\ref{ass:incr} and~\ref{ass:decrDeriv}) rested bandit with Normal reward and unit variance that can be identical up to $\lfloor T^{\beta} \rfloor \le \frac{T}{2}$ pulls. For every learning policy $\pi$, it holds that:
	\begin{align*}
		\sup_{\bm{\mu} \in \mathcal{N}_{2,\beta}} R_{\bm{\mu}}(\pi, T) \ge  \frac{\lfloor T^{\beta} + T^{\frac{2}{3}} \rfloor}{64 \sqrt{e}}.
	\end{align*}
\end{restatable}
The proof of Theorem~\ref{thr:apprendiamopoco} uses the same instances provided in Figure~\ref{fig:instance4} we used for Corollary~\ref{thr:Learnablebeta}.
Notice that the case for $\beta = 1$ is still covered by Theorem~\ref{thr:nonLearnable2} so that, even in the stochastic setting, the problem is unlearnable. The theorem also sheds light on the additional source of complexity given by the stochastic nature of the realizations. Indeed, delving deeper into the result, we see that in the stochastic case, due to the new source of error, an additional term of order $T^{\frac{2}{3}}$ arises. This implies that even in the simplest settings, i.e., when $\beta = 0$, differently from what would happen in the deterministic setting, we have at least a regret of order $T^{\frac{2}{3}}$.

\subsection[Upsilon-Dependent Regret Lower Bound]{$\Upsilon_{\bm{\mu}}$-Dependent Regret Lower Bound}\label{sec:ups}

In this section, we derive an instance-dependent regret lower bound, where the instance is described by the cumulative increment function $\Upsilon_{\bm{\mu}}(\cdot, \cdot)$.

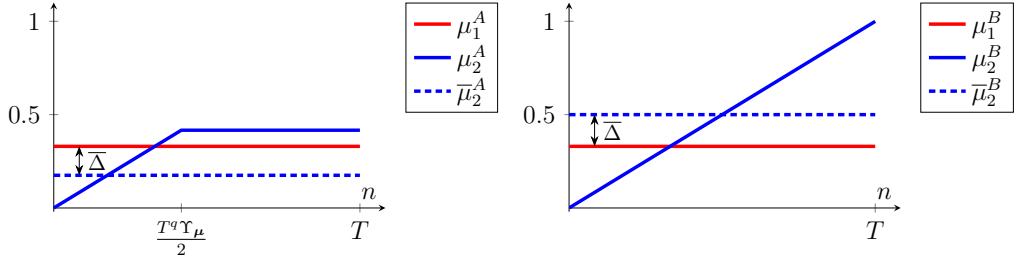
\begin{figure}[t]
	\begin{center}
		\resizebox{\textwidth}{!}{
			\begin{tikzpicture}[
				declare function={
					func(\x)= (\x < 5) * (0.0833 * \x) + (\x >= 5) * (.4165);
				}
				]
				\begin{groupplot}[
					group style={
						group size=2 by 1,
						vertical sep=0pt,
						horizontal sep=3cm,
						group name=G},
					width=7cm,height=5cm,
					axis lines=middle,
					legend style={at={(1.2,1)},anchor=north,legend cell align=left} 
					]
					\nextgroupplot[xmin=0, xmax=13, ymax=1.1, ymin=-0.02, domain=0:12, samples=1000, xtick = {0,5,12}, xticklabels = {$0$, $\frac{T^q\Upsilon_{\bm{\mu}}}{2}$, $T$}, xlabel={$n$}]
					\addplot[red,  ultra thick] (x,0.33);\addlegendentry{$\mu^A_1$}
					\addplot[blue,  ultra thick] {func(x)};\addlegendentry{$\mu^A_2$}
					\addplot[blue,  ultra thick, densely dashed] (x,0.175);\addlegendentry{$\overline{\mu}^A_2$}
					\draw [ <->, >=Stealth] (1,0.175) -- (1,0.33) node[pos=0.5,right]{\footnotesize $\overline{\Delta}$};
					
					\nextgroupplot[xmin=0, xmax=13, ymax=1.1, ymin=-0.02, domain=0:12, samples=1000, xtick = {0,12}, xticklabels = {$0$, $T$},xlabel={$n$}]
					\addplot[red,  ultra thick] (x,0.33);\addlegendentry{$\mu^B_1$}
					\addplot[blue, ultra thick] (x,0.0833*x);\addlegendentry{$\mu^B_2$}
					\addplot[blue,  ultra thick, densely dashed] (x,0.5);\addlegendentry{$\overline{\mu}^B_2$}
					\draw [ <->, >=Stealth] (1,0.33) -- (1,0.5) node[pos=0.5,right]{\footnotesize $\overline{\Delta}$};
				\end{groupplot}
			\end{tikzpicture}
		}
	\end{center}
\caption{The two 2-armed instances (A on the left, B on the right) of rested bandits with non-decreasing expected rewards used in the proof of Theorem~\ref{thr:nonLearnable3}.}\label{fig:instance3}
\end{figure}

\begin{restatable}[$\Upsilon_{\bm{\mu}}$-Dependent Regret Lower Bound under Assumptions~\ref{ass:incr} and~\ref{ass:decrDeriv}]{thr}{nonLearnableInstanceDep}\label{thr:nonLearnable3}
	Let $q \in [0,1]$, $\overline{\Upsilon} \ge 0$,  and $\mathcal{M}_2(\overline{\Upsilon})$ be the set of $2$-armed rising (Assumptions~\ref{ass:incr} and~\ref{ass:decrDeriv}) deterministic rested bandit such that ${\Upsilon_{\bm{\mu}}(T,q)} \ge {\overline{\Upsilon}}$. For every learning policy $\pi$, it holds that:
	\begin{align*}
		\sup_{\bm{\mu} \in \mathcal{M}_2(\overline{\Upsilon})} \frac{R_{\bm{\mu}}(\pi, T)}{\Upsilon_{\bm{\mu}}(T,q)} \ge \frac{T^q}{8}. 
	\end{align*}
\end{restatable}

Some comments are in order. First of all, the set of rising bandit instances $\mathcal{M}_2(\overline{\Upsilon})$ includes all the bandits in which the complexity term $\Upsilon_{\bm{\mu}}(T,q)$ is
larger than a constant value $\overline{\Upsilon}$. Thus, the interpretation of Theorem~\ref{thr:nonLearnable3} is that
a dependence on $\Upsilon_{\bm{\mu}}(T,q)$ is unavoidable whatever subset of deterministic rising bandits is considered 
that contains instances with ${\Upsilon_{\bm{\mu}}(T,q)} \ge {\overline{\Upsilon}}$.
The instances that are employed to obtain such a result are depicted in Figure~\ref{fig:instance3}.

%
%
%
%

\section{The \algrested Algorithm}\label{sec:alg}

\begin{algorithm}[t]
	\begin{algorithmic}
	\STATE \textbf{Input}: number of rounds $K$, optimistic indexes $(B_i)_{i \in [K]}$
	\STATE Initialize $N_{i} \leftarrow 0$ for all $i \in [K]$
	\FOR{$t \in [T]$}{
		\STATE Pull $I_t \in \argmax_{i \in [K]} \{B_i(t)\}$
		\STATE Observe $R_t \sim \nu_{I_t}(t, N_{I_t}+1)$
		\STATE Update $B_{I_t}$ and $N_{I_t} \leftarrow N_{I_t} + 1$ 
	}
	\ENDFOR
	\end{algorithmic}
	\caption{\texttt{R-ed-UCB}}\label{alg:alg}
\end{algorithm}

In this section, we devise and analyze learning algorithms for rested rising bandits for both the deterministic (Section~\ref{sec:restedDet}) and stochastic  (Section~\ref{sec:restedStoc}) settings. We will present an \emph{optimistic} algorithm, \texttt{Rising Rested Upper Confidence Bounds} (\algrested),  whose structure is summarized in Algorithm~\ref{alg:alg} and parametrized by an exploration index $B_i(t)$ that will be designed case by case.

\subsection{Deterministic Setting}\label{sec:restedDet}
To progressively introduce the core ideas, we begin with the case of deterministic arms ($\sigma = 0$). 
We devise an optimistic estimator of $\mu_i(t)$, namely $\overline{\mu}_i^{\text{\red}} (t)$, having observed the exact expected rewards $(\mu_i(n))_{n=1}^{N_{i,t-1}}$. Differently from the rotting setting, these expected rewards are an underestimation of $\mu_i(t)$. Therefore, we exploit the non-decreasing assumption (Assumption~\ref{ass:incr}) to derive the identity:
%
%
\begin{equation}\label{eq:estRestedDet}
\begin{aligned}
	\mu_i(t) =  \textcolor{color11}{\underbrace{\mu_i(N_{i,t-1})}_{\text{(most recent expected reward)}}} + \textcolor{color12}{\underbrace{\sum_{n=N_{i,t-1}}^{t-1} \gamma_i(n)\textcolor{black}{.}}_{\text{(sum of future increments)}}}
\end{aligned}
\end{equation}
By exploiting the concavity (Assumption~\ref{ass:decrDeriv}), we upper bound the sum of future increments with the last experienced increment $\gamma_i(N_{i,t-1}-1)$ that is projected for the future $t-N_{i,t-1}$ pulls, leading to the following estimator:
\begin{align}\label{eq:estRestedDetEst}
\overline{\mu}_i^{\text{\red}}(t) & \coloneqq \hspace{-0.22cm} \textcolor{color11}{\underbrace{\mu_i(N_{i,t-1})}_{\text{(most recent expected reward)}}} \hspace{-0.22cm} + \textcolor{color12}{(t - N_{i,t-1}) \hspace{-0.1cm} \underbrace{ \gamma_i(N_{i,t-1}-1)\textcolor{black}{,}}_{\text{(most recent increment)}}}
\end{align}
if $N_{i,t-1} \ge 2$ else $\overline{\mu}_i^{\text{\red}}(t) \coloneqq +\infty$. Figure~\ref{fig:estConstr} illustrates the construction of the estimator. The optimism of $\overline{\mu}_i^{\text{\red}}$ and a bias bound are proved in Lemma~\ref{lemma:lemmaDetRested}.

\begin{figure}
\centering
	\includegraphics[width=.7\linewidth]{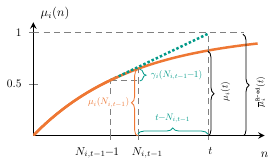}
	\caption{Graphical representation of the estimator construction $\overline{\mu}_i^{\text{\red}}(t)$ for the rested deterministic setting.}\label{fig:estConstr}
\end{figure}

%


\paragraph{Regret Analysis}
We are now ready to provide the regret analysis of \algrested, \ie Algorithm~\ref{alg:alg} when we employ as exploration index $B_i(t) \equiv \overline{\mu}_i^{\text{\red}}(t)$. 

\begin{restatable}[]{thr}{theRegretDetRested}\label{thr:theRegretDetRested}
Let $T \in \Nat$, then \algrested (Algorithm~\ref{alg:alg}) with $B_i(t) \equiv \overline{\mu}_i^{\text{\red}}(t)$ suffers an expected regret bounded, for every $q \in [0,1]$, as:
\begin{align*}
	R_{\bm{\mu}}(\text{\algrested},T) \le 2K + K T^q  \Upsilon_{\bm{\mu}}\left(\left\lceil \frac{T}{K} \right\rceil,q \right).
\end{align*}
\end{restatable}

The regret depends on a parameter $q \in [0,1]$ that can be selected to tighten the bound, whose optimal value depends on $\Upsilon_{\bm{\mu}}(\cdot,q)$, that is a function on the horizon $T$. Some examples, when $\gamma_i(t) \le l^{-c}$ for $c>0$, are reported in Figure~\ref{fig:rate}.

%
%
%

\subsection{Stochastic Setting}\label{sec:restedStoc}
Moving to the \red stochastic setting ($\sigma>0$), we cannot directly exploit the estimator in Equation~\eqref{eq:estRestedDetEst}. Indeed, we only observe the sequence of noisy rewards $(R_{t_{i,n}})_{n=1}^{N_{i,t-1}}$, where $t_{i,n} \in [T]$ is the round at which arm $i\in [K]$ was pulled for the $n$-th time. To cope with stochasticity, we need to employ an $h$-wide window made of the $h$ most recent samples, similarly to what has been proposed by~\citet{seznec2020asingle}. The choice of $h$ represents a \emph{bias-variance trade-off} between employing few recent observations (less biased), compared to many past observations (less variance). 
For $h \in [N_{i,t-1}]$, the resulting estimator $\widehat{\mu}_i^{\text{\red},h}(t)$ is given by:
\begin{align*}
	\widehat{\mu}_i^{\text{\red},h} (t)   \coloneqq \frac{1}{h}  \sum_{l=N_{i,t-1}-h+1}^{N_{i,t-1}} & \Bigg(\textcolor{color11}{\underbrace{R_{t_{i,l}}}_{\text{(estimated expected reward)}}} \hspace{-0.5 cm} + \textcolor{color12}{(t-l) \underbrace{\frac{R_{t_{i,l}} - R_{t_{i,l-h}}}{h}}_{\text{(estimated increment)}}}\Bigg),
\end{align*}
if $h \le \lfloor N_{i,t-1}/2 \rfloor$, else $\widehat{\mu}_i^{\text{\red},h} (t) \coloneqq +\infty$. The construction of the estimator is shown in Appendix~\ref{apx:prrested} and relies on the idea of averaging several estimators of the form of Equation~\eqref{eq:estRestedDetEst} instanced using as starting points different number of pulls $N_{i,t-1}-l+1$ for $l \in [h]$ and replacing the true expected reward with the corresponding reward sample. An efficient way to compute this estimator is reported in Appendix~\ref{apx:efficient}.

\paragraph{Regret Analysis}
By making use of the presented estimator, we build the following optimistic exploration index:
\begin{align*}
	& B_{i}(t) \equiv \widehat{\mu}_i^{\text{\red},h_{i,t}} (t) + \beta_i^{\text{\red},h_{i,t}}(t), \quad \text{where} \\
& \beta^{\text{\red},h_{i,t}}_i(t,\delta_t)\coloneqq \sigma (t-N_{i,t-1}+h_{i,t}-1) \sqrt{ \frac{ 10 \log \frac{1}{\delta_t} }{h_{i,t}^3} },
\end{align*}
and $h_{i,t}$ are arm-and-time-dependent window sizes and $\delta_t$ is a time-dependent confidence parameter. By choosing the window size depending linearly on the number of pulls, we are able to provide the following regret bound.

\begin{restatable}[]{thr}{theRegretRestedStochastic}\label{thr:theRegretRestedStochastic}
Let $T \in \Nat$, then \algrested (Algorithm~\ref{alg:alg}) with $B_i(t) \equiv  \widehat{\mu}_i^{\text{\red},h_{i,t}}(t) + \beta_i^{\text{\red},h_{i,t}}(t)$,  $h_{i,t} = \left\lfloor \epsilon N_{i,t-1}\right\rfloor$ for $\epsilon \in (0,1/2)$ and $\delta_t = t^{-\alpha}$ for $\alpha > 2$, suffers an expected regret bounded, for every $q \in \left[0,  1 \right]$, as:
\begin{align*}
	& R_{\bm{\mu}}(\text{\algrested},T) \hspace{-0.1 cm}\le \BigO \Bigg( \frac{K}{\epsilon} (\sigma  T)^{\frac{2}{3}} \left( \alpha  \log T \right)^{\frac{1}{3}} +   \frac{K {T}^q }{1-2\epsilon}  \Upsilon_{\bm{\mu}}\left(\left\lceil  (1-2\epsilon) \frac{T}{K} \right\rceil,q \right)\hspace{-0.2 cm}  \Bigg).
\end{align*}
\end{restatable}

\begin{figure}
	\centering
\renewcommand{\arraystretch}{2}
\resizebox{0.8\textwidth}{!}{%
\begin{tabular}{r|C{2cm}|C{2cm}|C{2cm}|C{2cm}|}
\cline{2-5}
$f(l)$ &  $e^{-cl}$ & $l^{-c}$ ($c>\frac{3}{2}$) & $l^{-c}$ ($1<c\le\frac{3}{2}$) & $l^{-c}$ ($c \le 1$) \\\cline{2-5}
  Deterministic & $\displaystyle \frac{K}{c} \log T$ & \multicolumn{2}{c|}{$\displaystyle KT^{\frac{1}{c}} \log T $} & $\displaystyle T $  \\\cline{2-5}
Stochastic & \multicolumn{2}{c|}{$\displaystyle KT^{\frac{2}{3}} (\log T)^{\frac{1}{3}} $}  & $\displaystyle KT^{\frac{1}{c}} \log T $ &  $\displaystyle T$ \\\cline{2-5}
\end{tabular}
}
\caption{Regret bounds $\widetilde{\mathcal{O}}$ rates optimized over $q$ for deterministic and stochastic rested rising bandits when $\gamma_i(l) \le l^{-c}$ for $c > 0$.}\label{fig:rate}
\end{figure}

This result deserves some comments. First, compared with the corresponding deterministic \red regret bound (Theorem~\ref{thr:theRegretDetRested}), it reflects a similar dependence of the cumulative increment $\Upsilon_{\bm{\mu}}$, although it now involves the $\epsilon$ parameter defining the window size $h_{i,t} = \lfloor \epsilon N_{i,t-1} \rfloor$.  Second, it includes an additional term of order $\widetilde{\mathcal{O}}(T^{\frac{2}{3}})$ that is due to the noise $\sigma$ presence that increases inversely \wrt the $\epsilon$.\footnote{\hl{In particular, when $\gamma_i(n)$ decreases sufficiently fast (see Table~\ref{tab:rates}), the regret is dominated by the $\widetilde{\mathcal{O}}(T^{\frac{2}{3}})$ component.}} Thus, we visualize a trade-off in the choice of $\epsilon$: larger windows ($\epsilon \approx 1$) are beneficial for the first term, but they enlarge the constant $1/(1-2\epsilon)$ multiplying the second component.

A final remark is about comparing the current work with adversarial bandits. The \red setting can be mapped to an \emph{adversarial} bandit~\cite{AuerCFS02} with an \emph{adaptive} (\ie non-oblivious) adversary. Indeed, the arm expected reward $\mu_i(N_{i,t})$ can be thought to as selected by an adversary who has access to the previous learner choices (\ie the history $\mathcal{H}_{t-1}$), specifically to the number of pulls $N_{i,t}$. However, although adversarial bandit algorithms, such as  \texttt{EXP3}~\citep{AuerCFS02} and \texttt{OSMD}~\citep{AudibertBL14}, suffer $\widetilde{{\BigO}}({\sqrt{T}})$ regret, these results are not comparable with ours. Indeed, while these correspond to guarantees on the \emph{external regret}, the regret definition we employ in Section~\ref{sec:problemSetting} is a notion of \emph{policy regret}~\citep{DekelTA12}.

\section{Numerical Simulations} \label{sec:experiments}

We numerically tested \algrested{} w.r.t.~state-of-the-art algorithms for non-stationary MABs in the \emph{rested} setting.\footnote{Details of the experimental setting are provided in Appendix~\ref{apx:experiments}. The code is available at \url{https://github.com/albertometelli/stochastic-rising-bandits}.}

We consider the following baseline algorithms: \texttt{Rexp3}~\citep{besbes2014stochastic}, a non-stationary MAB algorithm based on variation budget, \texttt{KL-UCB}~\citep{garivier2011kl}, one of the most effective stationary MAB algorithms, \texttt{Ser4}~\citep{allesiardo2017non}, which considers best arm switches during the process, and sliding-window algorithms that are generally able to deal with non-stationary restless bandit settings such as \texttt{SW-UCB}~\citep{garivier2011upper}, \texttt{SW-KL-UCB}~\citep{combes2014unimodal}, and \texttt{SW-TS}~\citep{trovo2020sliding}.
The parameters for all the baseline algorithms have been set as recommended in the corresponding papers (see also Appendix~\ref{apx:experiments}). For our algorithms, the window is set as $h_{i,t} = \lfloor \epsilon N_{i,t-1} \rfloor$ (as prescribed by Theorems~\ref{thr:theRegretRestedStochastic}). We remark that while the baseline algorithms are suited for the restless case, in the rested case, no algorithm has been designed to cope with the stochastic rising setting, provided that no knowledge on the expected reward function is available.
We compare the algorithms in terms of empirical cumulative regret $\widehat{R}_{\bm{\mu}}(\pi,t)$, which is the empirical counterpart of the expected cumulative regret $R_{\bm{\mu}}(\pi,t)$ at round $t$ averaged over multiple independent runs.

%

\begin{figure}
\centering
\includegraphics[width=0.65\textwidth]{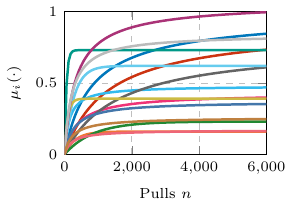}
\caption{$15$ arms setting: first $6000$ pulls of the expected reward functions $\mu_i(\cdot)$.} \label{fig:15arms_rewards}
\end{figure}

\begin{figure}
\centering
\includegraphics[width=0.9\textwidth]{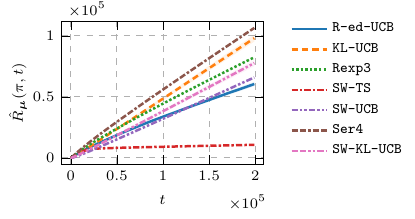}
\caption{$15$ arms setting: cumulative regret (shaded areas represents $95\%$~c.i. computed averaging over $100$ independent runs).} \label{fig:rested_15arms_regrets}
\end{figure}

\subsection{Rested setting}

We employ the same arms generated in the work by~\citet{metelli2022stochastic} to evaluate \algrested{} in the rested setting. We provide the arms expected rewards in Figure~\ref{fig:15arms_rewards} and the empirical cumulative regret of the analysed algorithms in Figure~\ref{fig:rested_15arms_regrets}. \texttt{SW-TS} is confirmed as the best algorithm at the end of the time horizon, although other algorithms (\texttt{SW-UCB} and \texttt{SW-KL-UCB}) suffer less regret at the beginning of learning. \algrested pays the price of the initial exploration, but at the end of the horizon, it manages to achieve the second-best performance. Notice that, besides \algrested, all other baseline algorithms are designed for the restless setting and are not endowed with any guarantee on the regret in the rested scenario.

To highlight this fact, we designed a particular $2$-arms rising rested bandit in which the optimal arm reveals only when pulled a sufficient number of times (linear in $T$). The expected reward functions, fulfilling Assumptions~\ref{ass:incr} and~\ref{ass:decrDeriv}, are shown in Figure~\ref{fig:2arms_rewards} and the algorithms empirical regrets in Figure~\ref{fig:rested_2arms_regrets}. Note that in this setting the expected (instantaneous) regret may be negative for $t < \frac{19T}{400}$, and this is the case for most of the algorithms for $t < 20,000$. While for the first $\approx 20,000$ rounds \algrested{} is on par with the other algorithms, it outperforms all the other policies over a longer run. Note that the regret for \texttt{Rexp3} and \texttt{Ser4} is decreasing the slope for $t > 40,000$, meaning that they are somehow reacting to the change in the reward of the two arms. \texttt{SW-TS} starts reacting even later, at around $t \approx 100,000$. However, they are not prompt to detect such a change in the rewards and, therefore, collect a large regret in the first part of the learning process. The other algorithms suffer a linear regret at the end of the time horizon since they do not employ forgetting mechanisms or because the sliding window should be tuned knowing the characteristics of the expected reward.

\begin{figure}
\centering
\includegraphics[width=0.6\textwidth]{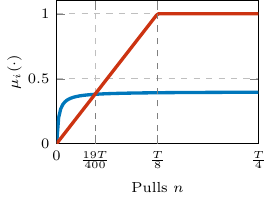}
\caption{$2$ arms setting: expected reward functions $\mu_i(\cdot)$ up to $T/4$.} \label{fig:2arms_rewards}
\end{figure}

\begin{figure}
\centering
\includegraphics[width=0.9\textwidth]{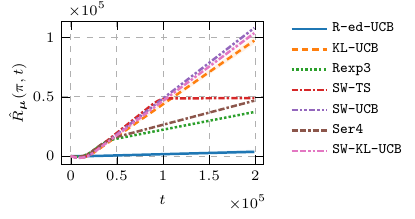}
\caption{$2$ arms setting: cumulative regret (shaded areas represents $95\%$ c.i.~computed averaging over $100$ independent runs).} \label{fig:rested_2arms_regrets}
\end{figure}

\begin{figure}
\centering
\includegraphics[width=0.8\textwidth]{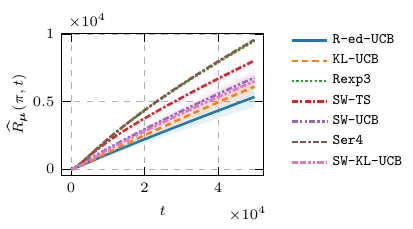} 
\caption{IMDB dataset setting: cumulative regret (shaded areas represents $95\%$~c.i. computed averaging over $30$ independent runs).}\label{fig:imdb}
\end{figure}  

\subsection{IMDB dataset (rested)} \label{sec:IMDB}
We investigate the performance of \algrested{} on an \emph{online model selection} task for a real-world dataset. We employ the IMDB dataset, made of $50,000$ reviews of movies (scores from $0$ to $10$). We preprocessed the data as done by~\citet{maas2011learning} to obtain a binary classification problem. Each review $\bm{x}_t$ lies in a $d = 10,000$ dimensional feature space, where each feature is the frequency of the most common English words. Each arm corresponds to a different online optimization algorithm, \ie two of them are Online Logistic Regression algorithms with different learning rate schemes, and the other five are Neural Networks with different topologies. We provide additional information on the arms of the bandit in Appendix~\ref{apx:imdb}.
At each round, a sample $\bm{x}_t$ is randomly selected from the dataset, a reward of $1$ is generated for correct classification, $0$ otherwise, and the online update step is performed for the chosen algorithm. 

The empirical regret is plotted in Figure~\ref{fig:imdb}. We can see that \algrested{}, with $\epsilon = 1/32$, outperforms the considered baselines. Compared to the synthetic simulations, the smaller window choice is justified by the fact that we need to take into account that the average learning curves of the classification algorithms are not guaranteed to be non-decreasing nor concave on a single run. However, keeping the window linear in $N_{i,t-1}$ is crucial for the regret guarantees of Theorem~\ref{thr:theRegretRestedStochastic}.

\section{Conclusions and Future Works}

This paper studied the MAB problem when the expected rewards are non-decreasing functions that evolve when pulling the corresponding arm, a.k.a.~rested bandit setting. We showed that an assumption on the expected reward (\eg concavity) is essential to make the problem learnable. However, we also showed that, in the most general case, the rising rested concave environment is still unlearnable and we inferred under which conditions it is possible to obtain a sublinear lower bound on the regret. We presented a novel algorithm that suitably employs the concavity assumption to build proper estimators for the setting. The algorithm is proven to suffer a regret made of a first instance-independent component of $\widetilde{\mathcal{O}}(T^{\frac{2}{3}})$ and an instance-dependent component involving the cumulative increment function $\Upsilon_{\bm{\mu}}(\cdot,q)$. For the rested setting, ours represents the first no-regret algorithm for the stochastic rising bandits. The experimental evaluation confirmed our theoretical findings showing advantages over state-of-the-art algorithms designed for non-stationary bandits.

Among the many interesting future lines of research starting from this work, we mention the possibility of designing effective algorithms if the increasing nature of the problem is somehow constrained and the study of cases in which the arms reward presents a structure among them. The former case includes the cases in which the increment is known to the learner, or we have stricter conditions than the concavity, \eg strong concavity. In such a setting, it is an open problem to show the lower bounds dependence and the design of algorithms explicitly exploiting such an additional structure. In the latter case, the research may consider standard structure among arms, \eg linear bandit or Gaussian process bandit structure, to exploit in a more efficient way the information coming from the entire set of arms, while still considering the rising nature of the problem.

\section*{Acknowledgments} 
This paper is supported by FAIR (Future Artificial Intelligence Research) project, funded by the NextGenerationEU program within the PNRR-PE-AI scheme (M4C2, Investment 1.3, Line on Artificial Intelligence).

\bibliography{biblio}
\bibliographystyle{elsarticle-num-names}

\newpage
\appendix

\setlength{\abovedisplayskip}{12pt}
\setlength{\belowdisplayskip}{12pt}
\setlength{\textfloatsep}{20pt}

\onecolumn 

\section{Proofs and Derivations}
In this section, we provide the proof of the results presented in the main paper.

\subsection{Proofs of Section~\ref{sec:problemSetting}}
\thrRestedOptimal*
\begin{proof}
	The proof is reported in Proposition 1 of~\cite{heidari2016tight}.
\end{proof}

\subsection{Proofs of Section~\ref{sec:lbs}}\label{apx:prrested}

\begin{lemma}\label{lemma:nonLearnableLemma}
	In the noiseless ($\sigma = 0$) setting, there exists a $2$-armed non-increasing non-concave bandit such that any learning policy $\pi$ suffers regret:
	\begin{align*}
		R_{\bm{\mu}}(\pi, T) \ge \left\lfloor \frac{T}{12} \right\rfloor.
	\end{align*}
\end{lemma}

\begin{proof}
Let $\bm{\mu}^A$ and $\bm{\mu}^B$ be two non-concave non-decreasing rested bandits, defined as:
\begin{align*}
	& \mu^A_1(n) = \mu^B_1(n) = \frac{1}{2},\\
	& \mu^A_2(n) = \begin{cases}
			0 & \text{if } n \le \lfloor \frac{T}{3} \rfloor \\ 
			1 & \text{otherwise}
		\end{cases} ,\\
		& \mu^B_2(n) = 0.
\end{align*}
Consider the instances of Figure~\ref{fig:instance1}. It is clear that for $\bm{\mu}^A$ the optimal arm is $2$, whereas for  bandit $\bm{\mu}^B$ the optimal arm is $1$, having optimal performance respectively $J_{\bm{\mu}^A}^*(T) = \lceil \frac{2}{3}T \rceil$ and $J_{\bm{\mu}^B}^*(T) = \frac{T}{2}$. Let $\pi$ be an arbitrary policy. Since the learner will receive the same rewards for both bandits until at least $\lfloor \frac{T}{3} \rfloor$. Thus, we have:
\begin{align*}
	\pi(\mathcal{H}_t(\bm{\mu}^A)) = \pi(\mathcal{H}_t(\bm{\mu}^B)) \; \implies \; \E_{\bm{\mu}^A} \left[ N_{1,\lfloor \frac{T}{3} \rfloor} \right] = \E_{\bm{\mu}^B} \left[ N_{1,\lfloor \frac{T}{3} \rfloor} \right] =: n_1.
\end{align*}
Let us now compute the performance of policy $\pi$ in the two bandits and the corresponding regrets. Let us start with $\bm{\mu}^A$:
\begin{align}
	J_{\bm{\mu}^A}(\pi,T) & = \frac{1}{2}\E_{\bm{\mu}^A}[N_{1,T}] + \max \left\{0, \E_{\bm{\mu}^A}[N_{2,T}] - \left\lfloor \frac{T}{3} \right\rfloor \right\} \label{p:001}\\
	& = \frac{1}{2}\E_{\bm{\mu}^A}[N_{1,T}] + \max \left\{0, \left\lceil \frac{2}{3}T \right\rceil - \E_{\bm{\mu}^A}[N_{1,T}] \right\},\label{p:002}
\end{align}
where Equation~\eqref{p:001} follows from observing that we get reward from arm $2$ only if we pull it more than $\lfloor\frac{T}{3}\rfloor$ times and Equation~\eqref{p:002} derives from observing that $T = \E_{\bm{\mu}^A}[N_{1,T}] + \E_{\bm{\mu}^A}[N_{2,T}]$. Now, consider the two cases:
\paragraph{Case (i)} $\E_{\bm{\mu}^A}[N_{1,T}] \ge \lceil \frac{2}{3} T \rceil$
\begin{align*}
	J_{\bm{\mu}^A}(\pi,T) = \frac{1}{2} \E_{\bm{\mu}^A}[N_{1,T}],
\end{align*}
that is maximized by taking $\E_{\bm{\mu}^A}[N_{1,T}] = T$. 
\paragraph{Case (ii)} $\E_{\bm{\mu}^A}[N_{1,T}] < \lceil \frac{2}{3} T\rceil$
\begin{align*}
	J_{\bm{\mu}^A}(\pi,T) =\left\lceil \frac{2}{3}T \right\rceil - \frac{1}{2} \E_{\bm{\mu}^A}[N_{1,T}],
\end{align*}
that is maximized by taking the minimum value of $\E_{\bm{\mu}^A}[N_{1,T}]$ possible, that is $\E_{\bm{\mu}^A}[N_{1,T}] \ge \E_{\bm{\mu}^A}[N_{1,\lfloor\frac{T}{3} \rfloor}] = n_1$. Putting all together, we have:
\begin{align*}
	J_{\bm{\mu}^A}(\pi,T) \le \max \left\{ \frac{T}{2}, \left\lceil\frac{2}{3}T\right\rceil - \frac{n_1}{2} \right\} =  \left\lceil \frac{2}{3}T \right\rceil - \frac{n_1}{2},
\end{align*}
having observed that $n_1 \le \lfloor\frac{T}{3}\rfloor$. Let us now focus on the regret:
\begin{align*}
	R_{\bm{\mu}^A}(\pi, T) = J_{\bm{\mu}^A}^*(T) - J_{\bm{\mu}^A}(\pi,T)  =\left\lceil \frac{2}{3}T \right\rceil -\left\lceil \frac{2}{3}T \right\rceil + \frac{n_1}{2} =  \frac{n_1}{2}.
\end{align*}
Consider now bandit $\bm{\mu}^B$, we have:
\begin{align*}
	J_{\bm{\mu}^B}(\pi,T) =  \frac{1}{2}  \E_{\bm{\mu}^B}[N_{1,T}] \le \frac{n_1}{2} + \left\lfloor \frac{T}{3} \right\rfloor,
\end{align*}
having observed that:
\begin{align*}
    \E_{\bm{\mu}^B}[N_{1,T}] =  n_1 + \E_{\bm{\mu}^B}[N_{1,T}] -  \E_{\bm{\mu}^B}[N_{1,\left\lfloor \frac{T}{3} \right\rfloor}] \le n_1 + \left\lceil \frac{2}{3}T \right\rceil.
\end{align*}
Let us now compute the regret:
\begin{align*}
	R_{\bm{\mu}^B}(\pi, T) =  J_{\bm{\mu}^B}^*(T) - J_{\bm{\mu}^B}(\pi,T) = \frac{T}{2} -  \frac{n_1}{2} - \left\lfloor\frac{T}{3}\right\rfloor = \left\lceil\frac{T}{6}\right\rceil - \frac{n_1}{2}.
\end{align*}
Finally, the worst-case regret can be lower bounded as follows:
\begin{align*}
	\inf_{\pi} \sup_{\bm{\mu}} &  R_{\bm{\mu}}(\pi,T) \ge \inf_{\pi} \max \left\{R_{\bm{\mu}^A}(\pi, T), R_{\bm{\mu}^B}(\pi, T) \right\} \\ & =\inf_{n_1 \in \left[ 0,\lfloor \frac{T}{3} \rfloor\right]} \max \left\{  \frac{n_1}{2} , \left\lceil \frac{T}{6} \right\rceil - \frac{n_1}{2}\right\} \ge \frac{1}{2} \left\lceil \frac{T}{6} \right\rceil \ge \left\lfloor \frac{T}{12} \right\rfloor,
\end{align*}
having minimized over $n_1$.
\end{proof}

\nonLearnable*
\begin{proof}
	It is sufficient to rescale the mean function of the proof of Lemma~\ref{lemma:nonLearnableLemma} by the quantity $\gamma_{\max}$.
\end{proof}

\nonLearnableLinear*
\begin{proof}
We have shown in \ref{lemma:Regret decomposition} that the online regret of any stochastic rising rested bandit environment $\bm{\mu}$ for any time horizon $T$ and for any fixed policy $\pi$ can be lower bounded as (we assume for notational convenience and w.l.o.g. that $i^\star(T)=1$):
    \begin{align}\label{rim:deco}
        R_{\bm{\mu}}(\pi, T)\ge\sum_{i \neq 1}\overline{\Delta}_i\mathbb{E}_{\bm{\mu},\pi}[N_{i,T}] 
    \end{align}
    The high-level idea is to devise two two-armed bandit environments that maximize the trade-off between the similarity and the competition, as shown in Figure~\ref{fig:instance2}. In order to do so we define the first two-arm instance $\bm{\mu}$ as:
    \begin{align}\label{eq:istanza1}
       \bm{\mu}=\begin{cases}
           \mu_1(n)= c &\text{constant with the number of pulls $n$}\\
           \mu_2(n)=\min\left\{\frac{n}{T},\frac{1}{2}\right\} & n \in \left[1,T\right]
       \end{cases},
    \end{align}
    while the second two-armed bandit environment $\bm{\mu}'$ is defined as:
   \begin{align}\label{eq:istanza2}
        \bm{\mu}'= \begin{cases}
           \mu_1'(n)= c &\text{constant with the number of pulls $n$}\\
           \mu_2'(n)=\frac{n}{T} & n \in \left[1,T\right]            
        \end{cases}.
    \end{align}
    As we want arm $2$ to be the optimal arm in the second instance, we define $c$ such that $\overline{\Delta}$ is identical in the two instances, i.e.:
    \begin{align}\label{eq:delta}
    2\overline{\Delta}=\frac{1}{T}\sum_{n=1}^T \left( \mu_2'(n)-\mu_2(n)\right),
    \end{align}
    and $c=\overline{\mu}_2(T)+\overline{\Delta}$. We provide now a lower bound on $\overline{\Delta}$ built as such. By straightforward calculation it is easy to show that: 
\begin{align}
 \overline{\mu}_2'(T)&=\frac{1}{2}+\frac{1}{2T},\\
 \overline{\mu}_2(T)&=\frac{1}{8}+\frac{1}{4}+\frac{1}{4T},
\end{align}    
so that:
\begin{align}
   2\overline{\Delta}\ge \frac{1}{8} \implies
   \overline{\Delta}\ge\frac{1}{16}.
\end{align}

From Lemma~\ref{lemma:Regret decomposition}, we can write the online regret in the two different instances with fixed policy $\pi$ as:
\begin{align}
    R_{\bm{\mu}}(\pi,T)&\ge \overline{\Delta}\mathbb{E}_{\bm{\mu},{\pi}}[N_{2,T}]\ge\frac{1}{16}\mathbb{E}_{\bm{\mu},{\pi}}[N_{2,T}], \\
    R_{\bm{\mu}'}(\pi,T)&\ge \overline{\Delta}\mathbb{E}_{\bm{\mu}',{\pi}}[N_{1,T}]\ge \frac{1}{16}\mathbb{E}_{\bm{\mu}',{\pi}}[N_{1,T}],
\end{align}
namely, in instance $\bm{\mu}$ the regret grows as we play the second arm whilst in the second instance $\bm{\mu'}$ as we play the first arm. Choosing as stopping time $\tau= \lfloor \frac{T}{2} \rfloor$, we have the two environment for every filtration up to $\tau$ are indistinguishable so that with some abuse of notation we can state that $\mathbb{E}_{\bm{\mu},{\pi}}\left[N_{1,\left\lfloor \frac{T}{2} \right\rfloor }\right] = \mathbb{E}_{\bm{\mu}',{\pi}}\left[N_{1,\left\lfloor \frac{T}{2} \right\rfloor}\right]$:
\begin{align}
    R_{\bm{\mu}}(\pi,T)+ R_{\bm{\mu}'}(\pi,T)& \ge  R_{\bm{\mu}}\left(\pi,\left\lfloor \frac{T}{2} \right\rfloor\right)+ R_{\bm{\mu}'}\left(\pi,\left\lfloor \frac{T}{2} \right\rfloor\right)\\
    &\ge \frac{1}{16} \left( \mathbb{E}_{\bm{\mu},{\pi}}\left[N_{2,\left\lfloor \frac{T}{2} \right\rfloor}\right]+\mathbb{E}_{\bm{\mu}',\pi}\left[N_{1,\left\lfloor \frac{T}{2} \right\rfloor}\right]\right)\\
    &\ge \frac{1}{16}\left( \mathbb{E}_{\bm{\mu},{\pi}}\left[N_{2,\left\lfloor \frac{T}{2} \right\rfloor}\right]+\mathbb{E}_{\bm{\mu},\pi}\left[N_{1,\left\lfloor \frac{T}{2} \right\rfloor}\right]\right)\\
    &\ge \left\lfloor \frac{T}{32} \right\rfloor.
\end{align} The proof follows from $\max\{a,b\}\ge \frac{a+b}{2}$, for $a,b \ge 0$.
\end{proof}
\clearpage

\LearnableLinearbeta*
\begin{proof}
The corollary directly follows rewriting the instances respectively as:
	\begin{align}\label{eq:istanzacor1}
		\bm{\mu}=\begin{cases}
			\mu_1(n)= c &\text{constant with the number of pulls $n$}\\
			\mu_2(n)=\min\left\{\frac{n}{T},\frac{\lfloor T^{\beta} \rfloor}{T}\right\} & n \in \left[1,T\right]
		\end{cases},
	\end{align}
and:
	\begin{align}\label{eq:istanzacor2}
		\bm{\mu}'= \begin{cases}
			\mu_1'(n)= c &\text{constant with the number of pulls $n$}\\
			\mu_2'(n)=\frac{n}{T} & n \in \left[1,T\right]            
		\end{cases}.
	\end{align}
A graphical representation of the instances is provided in Figure \ref{fig:instance4}. By definition of $\overline{\Delta}$ (see \eqref{eq:delta}), it is obvious that for any $T$ such that $\frac{T}{2}\ge \lfloor T^{\beta} \rfloor$ we have that $\overline{\Delta}\ge \frac{1}{16}$. Then, following the same step for the proof of Theorem \ref{thr:nonLearnable} we deduce that the instances are indistinguishable up to $\lfloor T^{\beta} \rfloor$, choosing a stopping time $\tau=\lfloor T^{\beta} \rfloor$ we can write with some abuse of notation  $\mathbb{E}_{\bm{\mu},{\pi}}\left[N_{1,\left\lfloor T^{\beta} \right\rfloor }\right] = \mathbb{E}_{\bm{\mu}',{\pi}}\left[N_{1,\left\lfloor T^{\beta} \right\rfloor}\right]$:
\begin{align}
	R_{\bm{\mu}}(\pi,T)+ R_{\bm{\mu}'}(\pi,T)& \ge  R_{\bm{\mu}}\left(\pi,\left\lfloor  T^{\beta} \right\rfloor\right)+ R_{\bm{\mu}'}\left(\pi,\left\lfloor T^{\beta} \right\rfloor\right)\\
	&\ge \frac{1}{16} \left( \mathbb{E}_{\bm{\mu},{\pi}}\left[N_{2,\left\lfloor T^{\beta} \right\rfloor}\right]+\mathbb{E}_{\bm{\mu}',\pi}\left[N_{1,\left\lfloor T^{\beta} \right\rfloor}\right]\right)\\
	&\ge \frac{1}{16}\left( \mathbb{E}_{\bm{\mu},{\pi}}\left[N_{2,\left\lfloor T^{\beta} \right\rfloor}\right]+\mathbb{E}_{\bm{\mu},\pi}\left[N_{1,\left\lfloor T^{\beta} \right\rfloor}\right]\right)\\
	&\ge  \left\lfloor \frac{T^{\beta}}{16} \right\rfloor,
\end{align}
the proof follows from $\max\{a,b\}\ge \frac{a+b}{2}$, for $a,b \ge 0$.
\end{proof}
\clearpage

\apprendiamopoco*
\begin{proof}
Using the same instances we have used for the deterministic setting (for reference see Figure \ref{fig:instance3}). Let us choose $\tau$, the stopping time for the adapted filtration, as $\tau=\lfloor T^{\beta} + T^{\frac{2}{3}} \rfloor$ whenever $\lfloor T^{\beta} + T^{\frac{2}{3}} \rfloor \le T$, then using Lemma \ref{lemma:Regret decomposition} we can rewrite the expected cumulative regret as:

\begin{align}
    R_{\bm{\mu}}(\pi,T) & +R_{\bm{\mu'}}(\pi,T)  \ge R_{\bm{\mu}} (\pi, \lfloor T^{\beta}  + T^{\frac{2}{3}} \rfloor)+R_{\bm{\mu'}} (\pi, \lfloor T^{\beta} + T^{\frac{2}{3}} \rfloor) \\
    & \ge \frac{\lfloor T^{\beta} + T^{\frac{2}{3}} \rfloor \overline{\Delta}}{2} \Bigg(\Pr_{\bm{\mu},\pi}\Big(N_1(\lfloor T^{\beta} + T^{\frac{2}{3}} \rfloor)  \leq \frac{\lfloor T^{\beta} + T^{\frac{2}{3}} \rfloor}{2}\Big)+\nonumber\\& \quad +\Pr_{\bm{\mu'},\pi}\Big(N_1(\lfloor T^{\beta} + T^{\frac{2}{3}} \rfloor) \ge \frac{\lfloor T^{\beta} + T^{\frac{2}{3}} \rfloor}{2}\Big)\Bigg)\\
    &\ge \frac{\lfloor T^{\beta} + T^{\frac{2}{3}} \rfloor\overline{\Delta}}{2}\exp\left(-\mathrm{D}\left(\mathbb{P}_{\bm{\mu} \mid \mathcal{F}_{\lfloor T^{\beta} + T^{\frac{2}{3}} \rfloor}},\mathbb{P}_{\bm{\mu'} \mid \mathcal{F}_{\lfloor T^{\beta} + T^{\frac{2}{3}} \rfloor}}\right)\right),
\end{align}
where for the last inequality we used the Bretagnolle-Huber inequality (Lemma \ref{lemma:bret}), with: 
$$
\mathrm{D}\left(\mathbb{P}_{\bm{\mu}\mid \mathcal{F}_{\lfloor T^{\beta} + T^{\frac{2}{3}} \rfloor}}, \mathbb{P}_{\bm{\mu^{\prime}}\mid \mathcal{F}_{\lfloor T^{\beta} + T^{\frac{2}{3}} \rfloor}}\right)=\mathbb{E}_{\bm{\mu}}\left[\log \left(\frac{\mathrm{d} \mathbb{P}_{\bm{\mu}\mid \mathcal{F}_{\lfloor T^{\beta} + T^{\frac{2}{3}} \rfloor}}}{\mathrm{d} \mathbb{P}_{\bm{\mu'}\mid \mathcal{F}_{\lfloor T^{\beta} + T^{\frac{2}{3}} \rfloor}}}\right)\right],
$$
with $\mathbb{P}_{\bm{\mu}\mid \mathcal{F}_{\lfloor T^{\beta} + T^{\frac{2}{3}} \rfloor}}$ and $\mathbb{P}_{\bm{\mu'}\mid \mathcal{F}_{\lfloor T^{\beta} + T^{\frac{2}{3}} \rfloor}}$ respectively  defined, with respect to the product measure $\left(\rho\times\lambda_T\right)^{\lfloor T^{\beta} + T^{\frac{2}{3}} \rfloor}$ defined for every measurable set $B$ as (see also \cite{lattimore2020bandit}, pages 64-65):
\begin{align}
	\mathbb{P}_{\bm{\mu}\mid \mathcal{F}_{\lfloor T^{\beta} + T^{\frac{2}{3}} \rfloor}}(B)=\int_{B}{p_{\bm{\mu} \pi}(\omega)\, d\left((\rho \times \lambda_T)^{\lfloor T^{\beta} + T^{\frac{2}{3}} \rfloor}(\omega)\right)}\\
	\mathbb{P}_{\bm{\mu'}\mid \mathcal{F}_{\lfloor T^{\beta} + T^{\frac{2}{3}} \rfloor}}(B)=\int_{B}{p_{\bm{\mu'} \pi}(\omega)\, d\left((\rho \times \lambda_T)^{\lfloor T^{\beta} + T^{\frac{2}{3}} \rfloor}(\omega)\right)},
\end{align}
being $\rho$ the counting measure and $\lambda_T$ the common measure defined as:
\begin{equation}\label{eq:commonmeasure}
	\lambda_T=\sum_{t=1}^T\sum_{i=1}^{K}\sum_{\underline{a}_t \in {[[K^{t-1}]]}}P_{i|\underline{a}_t}+P_{i|\underline{a}_t}',
\end{equation}
where the summation on $\underline{a}_t$ is over all the possible realizations of the actions up to $t-1$, $P_{i|\underline{a}_t}= P_{i|a_1,\ldots,a_{t-1}}$ is the conditional distribution of the reward of arm $i$ at time $t$ in the environment $\bm{\mu}$ and similarly $P'_{i|\underline{a}_t}=P'_{i|a_1,\ldots,a_{t-1}}$ is the the conditional distribution of the reward of arm $i$ in the environment $\bm{\mu'}$ (notice that the measures the different environments are conditioned to the same history of actions). For the scopes of our analysis, we will consider  $P_{i|a_1,\ldots,a_{t-1}}= \mathcal{N}(\mu_{i\mid a_1,\ldots,a_{t-1}},1)$ and $P_{i|a_1,\ldots,a_{t-1}}'= \mathcal{N}(\mu_{i\mid a_1,\ldots,a_{t-1}}',1)$, i.e., Normal distributions with unit variance. By definition of our instances (see Eq.\eqref{eq:istanza1} and Eq.\eqref{eq:istanza2} for $\beta=1$, Eq.\eqref{eq:istanzacor1} and Eq.\eqref{eq:istanzacor2} for $\beta<1$) $D(P_{i\mid a_1,\ldots,a_t-1},P_{i\mid a_1,\ldots,a_t-1}')<+\infty$ for every possible realization of the actions $a_1,\ldots,a_{t-1}$ then by definition $P_{i\mid a_1,\ldots,a_t-1}'$ is dominated by $P_{i\mid a_1,\ldots,a_t-1}$, i.e., $P_{i\mid a_1,\ldots,a_t-1}' \ll P_{i\mid a_1,\ldots,a_t-1}$ for all $\underline{a}_t$, $i$ and $t\in [[T]]$. Notice that, by definition, $\lambda_T$ dominates every measure $P_{i\mid a_1,\ldots,a_t-1}$ and  $P_{i\mid a_1,\ldots,a_t-1}'$, so formally holds that $P_{i\mid a_1,\ldots,a_t-1}\ll \lambda_T$ and $P_{i\mid a_1,\ldots,a_t-1}'\ll \lambda_T$ for any possible realization of $\underline{a}_t$, for all $i$ and $t \in [[T]]$. So, as for any finite time horizon $T$ the common measure $\lambda_T$ is a finite measure (as it is defined as the finite sum of finite measures), a density exists, see Lemma \ref{lemma:radnyd}, $ p_{a_t}\left(x_t\mid a_1, \ldots, a_{t-1}\right)$ , namely the Radon-Nykodim derivative, such that:
\begin{align}\label{eq:radon}
	\int_{B}{p_{a_t}\left(x_t\mid a_1, \ldots, a_{t-1}\right) d\lambda_T(x_t)}=P_{a_t|a_1,\ldots,a_{t-1}}(B),
\end{align}
and equivalently it is possible to define the density $p_{a_t}'\left(x_t\mid a_1, \ldots, a_{t-1}\right)$ such that:
\begin{align}
		\int_{B}{p_{a_t}'\left(x_t\mid a_1, \ldots, a_{t-1}\right) d\lambda_T(x_t)}=P_{a_t|a_1,\ldots,a_{t-1}}'(B).
\end{align}
Finally it's possible to define the density of $\mathbb{P}_{\bm{\mu}}$, i.e. $p_{\bm{\mu} \pi}$, as:
 \begin{align}
 	&p_{\bm{\mu} \pi}\left(a_1, x_1, \ldots, a_{\lfloor T^{\beta} + T^{\frac{2}{3}} \rfloor}, x_{\lfloor T^{\beta} + T^{\frac{2}{3}} \rfloor}\right)=\nonumber\\
 	&=\prod_{t=1}^{\lfloor T^{\beta} + T^{\frac{2}{3}} \rfloor} \pi_t\left(a_t \mid a_1, x_1, \ldots, a_{t-1}, x_{t-1}\right) p_{a_t}\left(x_t\mid a_1, \ldots, a_{t-1}\right).
 \end{align}

The density of $\mathbb{P}_{\bm{\mu^{\prime}}}$ is identical except that $p_{a_t}$ is replaced by $p_{a_t}^{\prime}$. Then, rewriting $\log \left(\frac{\mathrm{d} \mathbb{P}_{\bm{\mu}\mid \mathcal{F}_{\lfloor T^{\beta} + T^{\frac{2}{3}} \rfloor}}}{\mathrm{d} \mathbb{P}_{\bm{\mu'}\mid \mathcal{F}_{\lfloor T^{\beta} + T^{\frac{2}{3}} \rfloor}}}\right)$ using the chain rule for the Radon-Nykodim deriving w.r.t. the product measure $(\rho \times \lambda_T)^{\lfloor T^{\beta} + T^{\frac{2}{3}} \rfloor}$:

$$
\log \frac{d \mathbb{P}_{\bm{\mu} \mid\mathcal{F}_{\lfloor T^{\beta} + T^{\frac{2}{3}} \rfloor}}}{d \mathbb{P}_{\bm{\mu^{\prime}}\mid\mathcal{F}_{\lfloor T^{\beta} + T^{\frac{2}{3}} \rfloor}}}\left(a_1, x_1, \ldots, a_{t}, x_{t}\right)=\sum_{t=1}^{\lfloor T^{\beta} + T^{\frac{2}{3}} \rfloor} \log \frac{p_{a_t}\left(x_t\mid a_1, \ldots, a_{t-1}\right)}{p_{a_t}^{\prime}\left(x_t\mid a_1, \ldots, a_{t-1}\right)},
$$
where we used the chain rule for Radon-Nikodym derivatives and the fact that the terms involving the policy cancel out. Taking expectations of both sides:
\begin{align}
	&\mathbb{E}_{\bm{\mu}}\left[\log \frac{d \mathbb{P}_{\bm{\mu}\mid \mathcal{F}_{\lfloor T^{\beta} + T^{\frac{2}{3}} \rfloor}}}{d \mathbb{P}_{\bm{\mu^{\prime}}\mid \mathcal{F}_{\lfloor T^{\beta} + T^{\frac{2}{3}} \rfloor}}}\left(A_1, X_1, \ldots, A_{\lfloor T^{\beta} + T^{\frac{2}{3}} \rfloor}, X_{\lfloor T^{\beta} + T^{\frac{2}{3}} \rfloor}\right)\right]=\nonumber\\&=\sum_{t=1}^{\lfloor T^{\beta} + T^{\frac{2}{3}} \rfloor} \mathbb{E}_{\bm{\mu}}\left[\log \frac{p_{A_t}\left(X_t\mid A_1,  \ldots, A_{t-1}\right)}{p_{A_t}^{\prime}\left(X_t\mid A_1, \ldots, A_{t-1}\right)}\right].
\end{align}

Then we can write:
\begin{align}
    &\sum_{t=1}^{\lfloor T^{\beta} + T^{\frac{2}{3}} \rfloor}  \mathbb{E}_{\bm{\mu}}\left[\log \frac{p_{A_t}\left(X_t\mid A_1, \ldots, A_{t-1}\right)}{p_{A_t}^{\prime}\left(X_t\mid A_1, \ldots, A_{t-1}\right)}\right]= \nonumber \\
    & = \sum_{t=1}^{\lfloor T^{\beta} + T^{\frac{2}{3}} \rfloor} \mathbb{E}_{\bm{\mu}}\left[\overbrace{\mathbb{E}_{\bm{\mu}}\left[\log \frac{p_{A_t}\left(X_t\mid A_1, \ldots, A_{t-1}\right)}{p_{A_t}^{\prime}\left(X_t\mid A_1, \ldots, A_{t-1}\right)} \middle|\ A_t,\ldots,A_1\right]}^{(*)}\right]\label{eq:divergenzecoolback}\\
    &\leq \sum_{t=1}^{\lfloor T^{\beta} + T^{\frac{2}{3}} \rfloor} \mathbb{E}_{\bm{\mu}}\left[D(P_{A_t|A_1,\ldots,A_{t-1}},P'_{A_t|A_1,\ldots, A_{t-1}})\right],
\end{align}
where in \eqref{eq:divergenzecoolback} we used that under $\mathbb{P}_{\bm{\mu}|\mathcal{F}_{T^{\frac{2}{3}}+T^{\beta}}}(\cdot\mid A_t, \ldots, A_1)$ the distribution of the reward $X_t$ is given by definition by $dP_{A_t\mid A_1,\ldots,A_t}$ and the outer expected value $\mathbb{E}_{\bm{\mu}}[\cdot]$ is taken over all the possible realization of the actions under $\mathbb{P}_{\bm{\mu}|\mathcal{F}_{T^{\frac{2}{3}}+T^{\beta}}}(\cdot)$. Then again from the Radon–Nikodym derivative defined earlier (see Eq.\ref{eq:radon}) we have that $dP_{A_t\mid A_1,\ldots,A_t}=p_{A_t}\left(X_t\mid A_1, \ldots, A_{t-1}\right)d\lambda_T$, so $(*)$ can be rewritten as:
\begin{align}
	(*)&=\int{\log \frac{p_{A_t}\left(X_t\mid A_1, \ldots, A_{t-1}\right)}{p_{A_t}^{\prime}\left(X_t\mid A_1, \ldots, A_{t-1}\right)}dP_{A_t\mid A_1,\ldots,A_{t-1}}}\\
	&=\int{\log \frac{p_{A_t}\left(X_t\mid A_1,\ldots, A_{t-1}\right)}{p_{A_t}^{\prime}\left(X_t\mid A_1, \ldots, A_{t-1}\right)}p_{A_t}\left(X_t\mid A_1, \ldots, A_{t-1}\right)d\lambda_T}\\
	&=D(P_{A_t\mid A_1,\ldots,A_{t-1}},P'_{A_t\mid A_1,\ldots,A_{t-1}}).
\end{align}
Being the conditional rewards Normal-distributed with unit variance and considering the worst possible realization  $A_1, A_2,\ldots, A_t$ at any $t$ for the relative entropies of the distribution of the rewards of the arm $A_t$ in the two different environments when conditioned to the same history, namely $D(P_{A_t\mid A_1,\ldots,A_{t-1}},P'_{A_t\mid A_1,\ldots,A_{t-1}})$, we obtain for $t < \lfloor T^{\beta} \rfloor $:
\begin{align}
	D(P_{A_t|A_1,\ldots,A_{t-1}},P'_{A_t|A_1,\ldots, A_{t-1}})  = 0,
\end{align}
regardless the action $A_t$ and for $t \ge \lfloor T^{\beta} \rfloor $:
\begin{align}
   D(P_{A_t|A_1,\ldots,A_{t-1}},P'_{A_t|A_1,\ldots, A_{t-1}}) = \begin{cases}
        0 &\textit{if } A_t=1\\
        \frac{1}{2} \left(\frac{ (t- \lfloor T^{\beta} \rfloor )}{T}\right)^2 &\textit{if } A_t=2\\
    \end{cases},      
\end{align}
so that we can rewrite:
\begin{align}
   \sum_{t=1}^{\lfloor T^{\beta}+T^{\frac{2}{3}} \rfloor} \mathbb{E}_{\bm{\mu}}[ & D(P_{A_t|A_1,\ldots,A_{t-1}}, P'_{A_t|A_1,\ldots, A_{t-1}})]= \nonumber\\
   &= \sum_{t=\lfloor T^{\beta} \rfloor}^{\lfloor T^{\beta}+T^{\frac{2}{3}} \rfloor} \mathbb{E}_{\bm{\mu}}\left[D(P_{A_t|A_1,\ldots,A_{t-1}},P'_{A_t|A_1,\ldots,A_{t-1}})\right]\\
   &\leq \frac{1}{2T^2} \sum_{j=0}^{\lfloor T^{\frac{2}{3}}  \rfloor} j^2 \\
   & \le \frac{\lfloor T^{\frac{2}{3}}  \rfloor \lfloor T^{\frac{2}{3}}  \rfloor^2}{2T^2} \le \frac{1}{2}.
\end{align}
By substituting in what we have found earlier, remembering that for any $T$ such that $\frac{T}{2}\ge T^{\beta}$ the error at each round satisfies $\overline{\Delta}\ge\frac{1}{16}$, we obtain:
\begin{align}
    R_{\bm{\mu}}(\pi,T)+R_{\bm{\mu'}}(\pi,T) & \ge \frac{\lfloor T^{\beta} + T^{\frac{2}{3}} \rfloor}{32}\exp\left(-\frac{1}{2}\right),
\end{align}
the proof follows from $\max\{a,b\}\ge \frac{a+b}{2}$, for $a,b \ge 0$.
\end{proof}

\nonLearnableInstanceDep*
\begin{proof}
	We consider the two instances defined below for some $b,c\in (0,1]$ and $a \ge 1$ to be defined later:
	\begin{align*}
	\bm{\mu} = \begin{cases}
		\mu_1(n) = c \\
		\mu_2(n) =  \frac{n}{aT}
		\end{cases},
	\end{align*}
	and, for $\epsilon > 0$:
	\begin{align*}
	\bm{\mu}' = \begin{cases}
		\mu_1'(n) = c \\
		\mu_2'(n) = \min \left\{ \frac{n}{aT}, {b} \right\}
		\end{cases}.
	\end{align*}
	The instances are depicted in Figure~\ref{fig:instance3}. The points in which the arm 2 of the second instance become constant is given by 
	 $\tau' = {abT} \le T$, that implies the condition $ab \le 1$.
	Let us now compute the complexity indexes for the two cases:
	\begin{align*}
		& \Upsilon_{\bm{\mu}}(T,q) = \sum_{n=1}^T \left( \frac{1}{aT} \right)^q = a^{-q}T^{1-q}, \\
		& \Upsilon_{\bm{\mu}'}(T,q) = \sum_{n=1}^{\tau'} \left( \frac{1}{aT} \right)^q = (aT)^{1-q} b. \\
	\end{align*}
	Notice that $\Upsilon_{\bm{\mu}'}(T,q) \le \Upsilon_{\bm{\mu}}(T,q)$ and, clearly, $\min\{\Upsilon_{\bm{\mu}}(T,q),\Upsilon_{\bm{\mu}'}(T,q)\} = \Upsilon_{\bm{\mu}'}(T,q)$.
	Let us compute the average expected rewards for arms 2 of the two instances:
	\begin{align*}
		& \overline{\mu}_2(T) = \frac{T+1}{2aT}, \\
		&\overline{\mu}_2'(T) =\frac{b}{2T} \left(1 + (2-ab)T \right).
	\end{align*}
	We now enforce the value of $\overline{\Delta}$ such that $2\overline{\Delta} = \overline{\mu}_2(T) - \overline{\mu}_2'(T)$ and that will allow defining the value of $c$ by solving the equation $c =\overline{\mu}_2'(T)+  \overline{\Delta}$. We have:
	\begin{align*}
	2\overline{\Delta} = \frac{(1-ab)(1+T(1-ab))}{2aT} \ge \frac{\frac{1}{2}\left(1+\frac{T}{2}\right)}{2aT} \ge \frac{1}{4aT},
	\end{align*}
	having enforced $ab \le \frac{1}{2}$. We are now ready to lower bound the regret by exploiting Lemma~\ref{lemma:Regret decomposition}:
	\begin{align*}
	R_{\bm{\mu}}(\pi,T) \ge \overline{\Delta} \mathbb{E}_{\bm{\mu},\pi}[N_{2,T}], \qquad R_{\bm{\mu}'}(\pi,T) \ge \overline{\Delta} \mathbb{E}_{\bm{\mu}',\pi}[N_{1,T}].
	\end{align*}
	Thus, recalling that the two instances are indistinguishable until when $t \le \lfloor \tau' \rfloor$, we have:
	\begin{align*}
	R_{\bm{\mu}}(\pi,T) + R_{\bm{\mu}'}(\pi,T) & \ge R_{\bm{\mu}}(\pi,\lfloor \tau' \rfloor) + R_{\bm{\mu}'}(\pi,\lfloor \tau' \rfloor) \\
	& = \overline{\Delta} \left(\mathbb{E}_{\bm{\mu},\pi}[N_{2,\lfloor \tau' \rfloor}] +  \mathbb{E}_{\bm{\mu}',\pi}[N_{1,\lfloor \tau' \rfloor}] \right) \\
	& = \overline{\Delta} \left(\mathbb{E}_{\bm{\mu},\pi}[N_{2,\lfloor \tau' \rfloor}] +  \mathbb{E}_{\bm{\mu},\pi}[N_{1,\lfloor \tau' \rfloor}] \right)\\
	& = \overline{\Delta} \lfloor \tau' \rfloor \ge \frac{bT}{4}.
	\end{align*} 
	We now need to manipulate the last expression and highlight the dependence on the desired quantities:
	\begin{align*}
		\frac{bT}{4} = \underbrace{(aT)^{1-q} b}_{\min\{\Upsilon_{\bm{\mu}}(T,q),\Upsilon_{\bm{\mu}'}(T,q)\}}  \cdot \frac{T^q a^{q-1}}{4}.
	\end{align*}
	We choose $a = 1$ and:
	\begin{align*}
			b = \frac{\Upsilon_{\bm{\mu}}(T,q)}{2T^{1-q}} \le \frac{1}{2}.
	\end{align*}
	Putting all together, we have:
	\begin{align*}
	\max\left\{ \frac{R_{\bm{\mu}}(\pi,T)}{\Upsilon_{\bm{\mu}}(T,q)}, \frac{R_{\bm{\mu}'}(\pi,T)}{\Upsilon_{\bm{\mu}'}(T,q)} \right\} & \ge \frac{\max\{R_{\bm{\mu}}(\pi,T) , R_{\bm{\mu}'}(\pi,T)\}}{\min\{\Upsilon_{\bm{\mu}}(T,q),\Upsilon_{\bm{\mu}'}(T,q)\}} \\
	& \ge \frac{R_{\bm{\mu}}(\pi,T) + R_{\bm{\mu}'}(\pi,T)}{2\min\{\Upsilon_{\bm{\mu}}(T,q),\Upsilon_{\bm{\mu}'}(T,q)\}}  \\
	& \ge \frac{T^q}{8}.
	\end{align*}
	The statement follows.
\end{proof}

\subsection{Proofs of Section~\ref{sec:alg}}

\begin{restatable}[]{lemma}{lemmaDetRested}\label{lemma:lemmaDetRested}
	For every arm $i\in [K]$ and every round $t \in [T]$, let us define:
	\begin{align*}
		\overline{\mu}_i^{\text{\red}}(t) \coloneqq \mu_i(N_{i,t-1}) + (t - N_{i,t-1}) \gamma_i(N_{i,t-1}-1),
	\end{align*}
	if $N_{i,t-1} \ge 2$ else $\overline{\mu}_i^{\text{\red}}(t) \coloneqq +\infty$. Then, $\overline{\mu}_i^{\text{\red}}(t) \ge \mu_i(t) $	and, if $N_{i,t-1} \ge 2$, it holds that:
	\begin{align*}
		\overline{\mu}_i^{\text{\red}}(t) - \mu_i(N_{i,t}) \le (t-N_{i,t-1})  \gamma_i(N_{i,t-1}-1).
	\end{align*}
\end{restatable}
\begin{proof}
	Let us consider the following derivation:
	\begin{align*}
		\mu_i(t) &= \mu_i(N_{i,t-1}) + \sum_{n=N_{i,t-1}}^{t-1} \gamma_i(n) \le \mu_i(N_{i,t-1}) + (t-N_{i,t-1}) \gamma_i(N_{i,t-1}-1) \nonumber\\&\eqqcolon \overline{\mu}^{\text{\red}}_i(t),
	\end{align*}
	where the inequality holds thanks to Assumption~\ref{ass:decrDeriv}, having observed that $\sum_{n=N_{i,t-1}}^{t-1} \gamma_i(n) \le (t-N_{i,t-1}) \gamma_i(N_{i,t-1}) \le (t-N_{i,t-1}) \gamma_i(N_{i,t-1}-1)$. 
	For the bias bound, when $N_{i,t-1} \ge 2$, we consider the following derivation:
	\begin{align*}
		\overline{\mu}^{\text{\red}}_i(t) - \mu_i(N_{i,t}) & = \mu_i(N_{i,t-1}) + (t-N_{i,t-1}) \gamma_i(N_{i,t-1}-1) - \mu_i(N_{i,t}) \nonumber\\&\le (t-N_{i,t-1}) \gamma_i(N_{i,t-1}-1).
	\end{align*}
	having observed that $ \mu_i(N_{i,t-1})\le  \mu_i(N_{i,t})$ by Assumption~\ref{ass:incr}.
\end{proof}

\theRegretDetRested*
\begin{proof}
We have to analyze the following expression:
\begin{align*}
	R_{\bm{\mu}}(\text{\algrested},T) = \sum_{t=1}^T \mu_{i^{\star}}(t) - \mu_{I_t}(N_{i,t}) ,
\end{align*}
where $i^{\star} \in \argmax_{i \in [K]}\left\{ \sum_{l \in [T]} \mu_i(l) \right\}$. We consider a term at a time, use $B_i(t)\equiv\overline{\mu}_i^{\text{\red}}(t)$, and we exploit the optimism, \ie $B_{i^{\star}}(t) \le B_{I_t}(t)$:
\begin{align*}
	\mu_{i^{\star}}(t) &- \mu_{I_t}(N_{I_t,t}) + B_{I_t}(t) - B_{I_t}(t)\le  \\ &\le  \min \left\{ 1, \underbrace{\mu_{i^{\star}}(t) - B_{i^{\star}}(t)}_{\le 0} + B_{I_t}(t) - \mu_{I_t}(N_{I_t,t}) \right\}\\
	&  \le \min \left\{ 1, B_{I_t}(t) - \mu_{I_t}(N_{I_t,t})\right\}.
\end{align*}
Now we work on the term inside the minimum when $N_{I_t,t-1} \ge 2$:
\begin{align*}
B_{I_t}(t) - \mu_{I_t}(N_{I_t,t}) & = \overline{\mu}_i^{\text{\red}}(t) - \mu_{I_t}(N_{I_t,t})
\le (t - N_{i,t-1}) \gamma_{I_t}( N_{i,t-1}-1),
\end{align*}
where the inequality follows from Lemma~\ref{lemma:lemmaDetRested}.
We are going to decompose the summation of this term over the $K$ arms:
\begin{align*}
R_{\bm{\mu}}(\text{\algrested},T) & \le \sum_{t=1}^T  \min \left\{ 1, (t - N_{i,t-1}) \gamma_{I_t}( N_{i,t-1}-1)\right\} \\
& \le 2K + \sum_{i \in [K]} \sum_{j=3}^{N_{i,T}} \min \left\{ 1,(t_{i,j} - (j-1)) \gamma_{i}( j-2)\right\},
\end{align*}
where $t_{i,j} \in [T]$ is the round at which arm $i \in [K]$ was pulled for the $j$-th time.
Now, $q \in [0,1]$, then for any $x \ge 0$ it holds that $\min\{1,x\} \le \min\{1,x\}^{q} \le x^{q}$. By applying this latter inequality to the inner summation, we get:
\begin{align*}
\sum_{j=3}^{N_{i,T}} \min \left\{ 1, (t_{i,j}-(j-1)) \gamma_i(j-2)  \right\} & \le \sum_{j=3}^{N_{i,T}} \min \left\{1, T \gamma_i(j-2)\right\} \nonumber\\ &\le T^q \sum_{j=3}^{N_{i,T}} \gamma_i(j-2)^q,
\end{align*}
having used $t_{i,j}-(j-1) \le T$.
Summing over the arms, we obtain:
{
\begin{align*}
	 T^q \sum_{i \in [K]}  \sum_{j=3}^{N_{i,T}} \gamma_i(j-2)^q \le T^q K \Upsilon_{\bm{\mu}}\left( \left\lceil \frac{T}{K} \right\rceil,q \right),
\end{align*}
}
where the last inequality is obtained from Lemma~\ref{lemma:boundUpsilonMax}.
\end{proof}

\paragraph{Estimator Construction for the Stochastic Rising Rested  Setting}
Before moving to the proofs, we provide some intuition behind the estimator construction. We start observing that for every $l\in \{2, \dots, N_{i,t-1}\}$, we have that:
\begin{align*}
	\mu_i(t) &= \textcolor{color11}{\underbrace{\mu_i(l)}_{\text{(past expected reward)}}} + \textcolor{color12}{\underbrace{\sum_{j=l}^{t-1} \gamma_i(j)}_{\text{(sum of future increments)}}}  \nonumber\\ &\le\textcolor{color11}{\underbrace{\mu_i(l)}_{\text{(past expected reward)}}} + \textcolor{color12}{(t-l)\underbrace{\gamma_i(l-1)}_{\text{(past increment)}}},
\end{align*}
where the inequality follows from Assumption~\ref{ass:decrDeriv}.\footnote{The estimator of the deterministic case in Equation~\eqref{eq:estRestedDetEst} is obtained by setting $l=N_{i,t-1}$.} Since we do not have access to the exact expected rewards $\mu_i(l)$ and exact increments $\gamma_i(l-1) = \mu_i(l) - \mu_i(l-1)$, one may be tempted to directly replace them with the corresponding point estimates $R_{t_{i,l}}$ and $R_{t_{i,l}} - R_{t_{i,l-1}}$ and average the resulting estimators for a window of the most recent $h$ values of $l$. Unfortunately, while replacing $\mu_i(l)$ with $R_{t_{i,l}}$ is a viable option, replacing $\gamma_i(l-1)$ with $R_{t_{i,l}} - R_{t_{i,l-1}}$ will prevent concentration since the estimate  $R_{i,t_l} - R_{i,t_{l-1}}$ is too unstable. To this end, before moving to the estimator, we need a further bounding step to get a more stable, although looser, quantity. Based on Lemma~\ref{lemma:boundDeriv}, we bound for every $l\in \{2, \dots, N_{i,t-1}\}$ and $h \in[ l -1]$:
\begin{align*}
\textcolor{color12}{\underbrace{\gamma_i(l-1)}_{\text{(past increment at $l$)}}} \le \textcolor{color12}{\underbrace{\frac{\mu_i(l) - \mu_i(l-h)}{h}}_{\text{(average past increment over $\{l-h,\dots,l\}$)}}}.
\end{align*}
We can now introduce the optimistic approximation of $\mu_i(t)$, \ie $\widetilde{\mu}_i^{\text{\red},h}(t)$, and the corresponding estimator, \ie $\widehat{\mu}_i^{\text{\red},h}(t)$, that are defined in terms of a window of size $1 \le h \le \lfloor N_{i,t-1}/2 \rfloor$:
\begin{align*}
&  \widetilde{\mu}_i^{\text{\red},h} (t)  \coloneqq \frac{1}{h} \sum_{l=N_{i,t-1}-h+1}^{N_{i,t-1}} \Bigg(\textcolor{color11}{\underbrace{\mu_i(l)}_{\text{(past expected reward)}}} + \textcolor{color12}{(t-l) \underbrace{\frac{\mu_i(l) - \mu_{i}(l-h)}{h}}_{\text{(average past increment)}}}\Bigg),\\
& \widehat{\mu}_i^{\text{\red},h} (t)  \coloneqq \frac{1}{h} \sum_{l=N_{i,t-1}-h+1}^{N_{i,t-1}} \Bigg(\textcolor{color11}{\underbrace{R_{t_{i,l}}}_{\text{(estimated past expected reward)}}} \nonumber\\  & \hspace{6.5 cm}+ \textcolor{color12}{(t-l) \underbrace{\frac{R_{t_{i,l}} - R_{t_{i,l-h}}}{h}}_{\text{(estimated average past increment)}}}\Bigg).
\end{align*}

The proof is composed of the following steps:
\begin{enumerate}[label=(\roman*)]
	\item Lemma~\ref{lemma:firstEstimator} shows that $\widetilde{\mu}_i^{\text{\red},h}(t)$ is an upper-bound for $\mu_i(t)$ and provides a bound to its bias \wrt $\mu_i(N_{i,t})$ for every value of $h$;
	\item Lemma~\ref{lemma:firstEstimatorConc} analyzes the concentration of $\widehat{\mu}_i^{\text{\red},h}(t)$ around $\widetilde{\mu}_i^{\text{\red},h}(t)$ for a specific choice of $\delta_t = t^{-\alpha}$ and when $h_{i,h} \coloneqq h(N_{i,t-1})$ is a function of the number of pulls $N_{i,t-1}$ only;
	\item Theorem~\ref{thr:theRegretRestedStochastic} bounds the expected regret of \algrested when $h_{i,h} =  \lfloor \epsilon N_{i,t-1} \rfloor$, for $\epsilon \in (0,1/2)$.
\end{enumerate}

\begin{restatable}[]{lemma}{firstEstimator}\label{lemma:firstEstimator}
For every arm $i \in [K]$, every round $t \in [T]$, and window width $ 1 \le h \le \lfloor N_{i,t-1}/2 \rfloor$, let us define:
\begin{align*}
	 \widetilde{\mu}_i^{\text{\red},h} (t)  \coloneqq \frac{1}{h} \sum_{l=N_{i,t-1}-h+1}^{N_{i,t-1}} \left(\mu_i(l) + (t-l) \frac{\mu_i(l) - \mu_{i}(l-h)}{h}\right),
\end{align*}
otherwise if $h=0$, we set $\widetilde{\mu}_i^{\text{\red},h}(t)  \coloneqq +\infty$. Then, $\widetilde{\mu}_i^{\text{\red},h}(t) \ge \mu_i(t)$ and, if $N_{i,t-1} \ge 2$, it holds that:
\begin{align*}
	& \widetilde{\mu}_i^{\text{\red},h}(t)   - \mu_i(N_{i,t}) \le \frac{1}{2}(2t - 2N_{i,t-1} + h -1) \gamma_i(N_{i,t-1}-2h+1).
\end{align*}
\end{restatable}
\begin{proof}
Following the derivation provided above, we have that for every $l\in \{2, \dots, N_{i,t-1}\}$:
\begin{align}
\mu_i(t) & = {\mu_i(l)} + \sum_{j=l}^{t-1} \gamma_i(j) \notag\\
&  \le  \mu_i(l) + (t-l){\gamma_i(l-1)} \label{:pp:-001}\\
& \le  \mu_i(l) + (t-l)\frac{\mu_i(l) - \mu_i(l-h)}{h}, \label{:pp:-002}
\end{align}
where line~\eqref{:pp:-001} follows from Assumption~\ref{ass:decrDeriv}, line~\eqref{:pp:-002} is obtained from Lemma~\ref{lemma:boundDeriv}. By averaging over the most recent $1 \le h \le \lfloor N_{i,t-1}/2 \rfloor$ pulls, we obtain:
\begin{align*}
	\mu_i(t) &\le \frac{1}{h} \sum_{l=N_{i,t-1}-h+1}^{N_{i,t-1}} \left(\mu_i(l) + (t-l)\frac{\mu_i(l) - \mu_i(l-h)}{h} \right) \eqqcolon \widetilde{\mu}_i^{\text{\red},h}(t).
\end{align*}
For the bias bound, when $N_{i,t-1} 	\ge 2$, we have:
\begin{align}
&\widetilde{\mu}_i^{\text{\red},h}(t) - \mu_i(N_{i,t})  = \nonumber\\ &=\frac{1}{h} \sum_{l=N_{i,t-1}-h+1}^{N_{i,t-1}}  \left(\mu_i(l) +  (t-l)  \frac{\mu_i(l) - \mu_i(l-h)}{h} \right) - \mu_i(N_{i,t}) \label{:pp:-0032}\\
& \le \frac{1}{h} \sum_{l=N_{i,t-1}-h+1}^{N_{i,t-1}} (t-l)  \frac{\mu_i(l) - \mu_i(l-h)}{h} \notag\\
& = \frac{1}{h} \sum_{l=N_{i,t-1}-h+1}^{N_{i,t-1}} (t-l)  \frac{1}{h} \sum_{j=l-h}^{l-1} \gamma_j(l)\notag \\
& \le \frac{1}{h} \sum_{l=N_{i,t-1}-h+1}^{N_{i,t-1}} (t-l) \gamma_i(l-h) \label{:pp:-003}\\
& \le \frac{1}{2}(2t - 2N_{i,t-1} + h -1) \gamma_i(N_{i,t-1}-2h+1) \label{:pp:-004}.
\end{align}
where line~\eqref{:pp:-0032} follows from Assumption~\ref{ass:incr} applied as $\mu_i(l) \le \mu_i(N_{i,t})$, line~\eqref{:pp:-003} follows from Assumption~\ref{ass:decrDeriv} and bounding $ \frac{1}{h} \sum_{j=l-h}^{l-1} \gamma_j(l) \le \gamma_i(l-h)$ and line~\eqref{:pp:-004} is derived still from Assumption~\ref{ass:decrDeriv}, $\gamma_i(l-h) \le \gamma_i(N_{i,t-1}-2h+1)$ and computing the summation.
\end{proof}
\begin{restatable}[]{lemma}{firstEstimatorConc}\label{lemma:firstEstimatorConc}
For every arm $i \in [K]$, every round $t \in [T]$, and window width $1 \le h \le \lfloor N_{i,t-1} / 2\rfloor$, let us define:
\begin{align*}
	 & \widehat{\mu}_i^{\text{\red},h} (t)  \coloneqq \frac{1}{h} \sum_{l=N_{i,t-1}-h+1}^{N_{i,t-1}} \left(R_{t_{i,l}} + (t-l) \frac{R_{t_{i,l}} - R_{t_{i,l-h}}}{h}\right),\\
	 & \beta^{\text{\red},h}_i(t,\delta)\coloneqq \sigma  (t-N_{i,t-1}+h-1) \sqrt{ \frac{ 10  \log \frac{1}{\delta} }{h^3} },
\end{align*}
otherwise if $h=0$, we set $\widehat{\mu}_i^{\text{\red},h} (t) \coloneqq +\infty$ and $\beta^{\text{\red},h}_i(t,\delta)\coloneqq+\infty$ .
Then, if the window size depends on the number of pulls only $h_{i,t} = h(N_{i,t-1}) $ and if $\delta_t = t^{-\alpha}$ for some $\alpha > 2$, it holds for every round  $t \in [T]$ that:
\begin{align*}
	\Pr \left(\left| \widehat{\mu}_i^{\text{\red},h_{i,t}}(t) - \widetilde{\mu}_i^{\text{\red},h_{i,t}}(t) \right| >  \beta^{\text{\red},h_{i,t}}_i(t,\delta_t) \right) \le 2t^{1-\alpha}.
\end{align*}
\end{restatable}

\begin{proof}
First of all, we observe under the event $\{h_{i,t} = 0\}$, then $\widehat{\mu}_i^{\text{\red},h_{i,t}}(t) = \widetilde{\mu}_i^{\text{\red},h_{i,t}}(t) = \beta^{\text{\red},h_{i,t}}_i(t,\delta_t)= +\infty$. By convening that $(+\infty) - (+\infty) = 0$, we have that $0 >  \beta^{\text{\red},h_{i,t}}_i(t,\delta_t)$ is not satisfied. Thus, we perform the analysis under the event $\{h_{i,t} \ge 1\}$. We first get rid of the dependence on the random number of pulls $N_{i,t-1}$:
%
\begin{align}
 &\Pr  \left(\left| \widehat{\mu}_i^{\text{\red},h_{i,t}}(t) - \widetilde{\mu}_i^{\text{\red},h_{i,t}}(t) \right| >  \beta^{\text{\red}, h_{i,t}}_i(t,\delta_t) \right) \notag \\
   &= \Pr \left(\left| \widehat{\mu}_i^{\text{\red},h(N_{i,t-1}) }(t) - \widetilde{\mu}_i^{\text{\red},h(N_{i,t-1}) }(t) \right| >  \beta^{\text{\red},h(N_{i,t-1}) }_i(t,\delta_t) \right) \label{:pp:-000} \\
   &\le \Pr \Big(\exists n \in \{0,\dots,t-1\} \text{ s.t. } h(n) \ge 1\,:\, \nonumber\\ & \hspace{3.5 cm}\left| \widehat{\mu}_i^{\text{\red},h(n)}(t) - \widetilde{\mu}_i^{\text{\red}, h(n)}(t) \right| >  \beta^{\text{\red},h(n)}_i(t,\delta_t) \Big) \notag \\
   &\le \sum_{n \in \{0,\dots,t-1\} \,:\, h(n) \ge 1} \Pr \left(\left| \widehat{\mu}_i^{\text{\red},h(n)}(t) - \widetilde{\mu}_i^{\text{\red},h(n)}(t) \right| >  \beta^{\text{\red},h(n)}_i(t,\delta_t) \right), \label{:pp:-010}
\end{align}
where line~\eqref{:pp:-000} derives from the definition of $h_{i,t} = h(N_{i,t-1})$ and  line~\eqref{:pp:-010} follows from a union bound over the possible values of $N_{i,t-1}$. Now, having fixed the value of $n$, we rewrite the quantity to be bounded:
\begin{align*}
h(n) & \left( \widehat{\mu}_i^{\text{\red},h(n)}(t) - \widetilde{\mu}_i^{\text{\red},h(n)}(t) \right) =\\ &= \sum_{l=n-h(n)+1}^{n} \left(X_l + (t-l) \frac{X_l - X_{l-h(n)}}{h(n)}\right)\\
&  = \sum_{l=n-h(n)+1}^{n}  \left(1 + \frac{t-l}{h(n)}\right) X_l - \sum_{l=n-h(n)+1}^{n}  \frac{t-l}{h(n)} \cdot X_{l-h(n)},
\end{align*}
where $X_{l} \coloneqq R_{t_{i,l}} - \mu_i(l)$. It is worth noting that we can index $X_l$ with the number of pulls $l$ only as the distribution of $R_{t_{i,l}}$ is fully determined by $l$ and $n$  (that are non-random quantities now) and, consequently, all variables $X_l$ and $X_{l-h(n)}$ are independent.

Now we apply Azuma-Ho\"effding's inequality of Lemma~\ref{lemma:hoeffding} for weighted sums of subgaussian martingale difference sequences. To this purpose, we compute the sum of the square weights:
\begin{align}
	\sum_{l=n-h(n)+1}^{n} \left(1+ \frac{t-l}{h(n)}\right)^2 & + \sum_{l=n-h(n)+1}^{n} \left( \frac{t-l}{h(n)} \right)^2\le \nonumber\\
	 & \le h(n) \left(1+ \frac{ t-n+h(n)-1}{h(n)}\right)^2 +\nonumber\\ & \hspace{1 cm}+ h(n)\left( \frac{ t-n+h(n)-1}{h(n)} \right)^2 \label{eq:prim}\\
	 &  \le \frac{ 5(t-n+h(n)-1)^2}{h(n)}, \label{eq:secon}
	 \end{align}
	 where line~\eqref{eq:prim} follows from bounding $t-l \le t-n+h(n)-1$ and line~\eqref{eq:secon} from observing that $\frac{ t-n+h(n)-1}{h(n)} \ge 1$. Thus, we have:
	 \begin{align*}
	 & \Pr \left(\left| \widehat{\mu}_i^{\text{\red},h(n)}(t) - \widetilde{\mu}_i^{\text{\red}, h(n)}(t) \right| >  \beta^{\text{\red},h(n)}_i(t,\delta_t) \right)\le \\
	 & \le \Pr \left(\left| \sum_{l=n-h(n)+1}^{n}  \hspace{-0.24 cm}\left(1 + \frac{t-l}{h(n)}\right) X_l - \hspace{-0.7 cm}\sum_{l=n-h(n)+1}^{n} \hspace{-0.4 cm} \frac{t-l}{h(n)}  X_{l-h(n)} \right| > h(n) \beta^{\text{\red},h(n)}_i(t,\delta_t) \right)\\
	 & \le 2 \exp \left( - \frac{\left(h(n) \beta^{\text{\red},h(n)}_i(t,\delta_t) \right)^2}{2 \sigma^2 \left( \frac{ 5(t-n+h(n)-1)^2}{h(n)} \right)}\right) = 2 \delta_t.
	 \end{align*}
	 By replacing this result into Equation~\eqref{:pp:-010}, and recalling the value of $\delta_t$, we obtain:
	 \begin{align*}
	 \sum_{n\in\{0,\dots,t-1\} \,:\, h(n) \ge 1}  2 \delta_t \le \sum_{n=0}^{t-1}  2 \delta_t = \sum_{n=0}^{t-1} 2t^{-\alpha} \le 2t^{1-\alpha}.
	 \end{align*}
\end{proof}

\theRegretRestedStochastic*
\begin{proof}
Let us define the good events $\mathcal{E}_t = \bigcap_{i \in [K]} \mathcal{E}_{i,t}$ that correspond to the event in which all confidence intervals hold:
\begin{align*} 
\mathcal{E}_{i,t} \coloneqq \left\{ \left| \widetilde{\mu}_i^{\text{\red},h_{i,t}}(t) - \widehat{\mu}_i^{\text{\red},h_{i,t}}(t)\right| \le \beta^{\text{\red},h_{i,t}}_i(t)\right\} \qquad \forall i \in [T], \,i \in [K]
\end{align*}
We have to analyze the following expression:
\begin{align*}
	R_{\bm{\mu}}(\text{\algrested},T) = \E \left[\sum_{t=1}^T \mu_{i^{\star}}(t) - \mu_{I_t}(N_{i,t})\right] ,
\end{align*}
where $i^{\star} \in \argmax_{i \in [K]}\left\{\sum_{l \in  [T]} \mu_i(l) \right\}$.  We decompose the above expression according to the good events $\mathcal{E}_t$:
\begin{align}
R_{\bm{\mu}}(\text{\algrested},T) & = \sum_{t=1}^T \mathbb{E} \left[\left( \mu_{i^{\star}}(t) - \mu_{I_t}(N_{I_t,t}) \right)\mathds{1}\{\mathcal{E}_t\} \right] + \nonumber\\&\hspace{2 cm}+\sum_{t=1}^T \mathbb{E} \left[\left( \mu_{i^{\star}}(t) - \mu_{I_t}(N_{I_t,t}) \right)\mathds{1}\{\lnot\mathcal{E}_t\}\right] \\
& \le \sum_{t=1}^T \mathbb{E} \left[\left( \mu_{i^{\star}}(t) - \mu_{I_t}(N_{I_t,t}) \right)\mathds{1}\{\mathcal{E}_t\}\right] + \sum_{t=1}^T \mathbb{E} \left[\mathds{1}\{\lnot\mathcal{E}_t\}\right], \label{eq:terz}
\end{align}
where we exploited $\mu_{i^{\star}}(t) - \mu_{I_t}(N_{I_t,t}) \le 1$ in line~\eqref{eq:terz}.
Now, we bound the second summation, recalling that $\alpha > 2$:
\begin{align*}
\sum_{t=1}^T \mathbb{E} \left[\mathds{1}\{\lnot\mathcal{E}_t\}\right] &= \sum_{t=1}^T \Pr\left(\lnot\mathcal{E}_t\right) =  1 + \sum_{t=2}^T \Pr\left( \lnot \bigcap_{i \in [K]}  \mathcal{E}_{i,t}\right) \\&\le 1 + \sum_{t=2}^T \Pr\left( \bigcup_{i \in [K]} \lnot \mathcal{E}_{i,t}\right) \le 1 + \sum_{i \in [K]} \sum_{t=2}^T \Pr\left( \lnot \mathcal{E}_{i,t}\right),
\end{align*}
where the first inequality is obtained with $\Pr(\lnot \mathcal{E}_{1}) \le 1$ and the second with a union bound over $[K]$. Recalling $\Pr(\lnot\mathcal{E}_{i,t})$ was bounded in Lemma~\ref{lemma:firstEstimatorConc}, we bound the summation with the integral as in Lemma~\ref{lemma:tecSum} to get:
\begin{align*}
	\sum_{i \in [K]} \sum_{t=2}^T \Pr\left( \lnot \mathcal{E}_{i,t}\right) \le \sum_{i \in [K]} \sum_{t=2}^T 2t^{1-\alpha} \le 2K \int_{x=1}^{+\infty} x^{1-\alpha} \de x = \frac{2K}{\alpha-2}.
\end{align*}
From now on, we proceed the analysis under the good events $\mathcal{E}_t$, recalling that $B_i(t) \equiv \widehat{\mu}_i^{\text{\red}, h_{i,t}}(t) + \beta_i^{\text{\red}, h_{i,t}}(t,\delta_t)$. 
We consider each addendum of the summation and we exploit the optimism, \ie $B_{i^{\star}}(t) \le B_{I_t}(t)$:
\begin{align*}
	\mu_{i^{\star}}(t) - \mu_{I_t}(N_{I_t,t}) &+ B_{I_t}(t) - B_{I_t}(t) \le \\ &\le  \min \left\{ 1, \underbrace{\mu_{i^{\star}}(t) - B_{i^{\star}}(t)}_{\le 0} + B_{I_t}(t) - \mu_{I_t}(N_{I_t,t}) \right\} \\
	& \le \min \left\{ 1, B_{I_t}(t) - \mu_{I_t}(N_{I_t,t})\right\}.
\end{align*}
Now, we work on the term inside the minimum:
\begin{align}
	B_{I_t}(t) - \mu_{I_t}(N_{I_t,t}) &= \widehat{\mu}_{I_t}^{\text{\red},h_{I_t,t}}(t)  + \beta^{\text{\red},h_{I_t,t}}_{I_t}(t,\delta_t)-\mu_{I_t}(N_{I_t,t}) \label{:pp:-2001} \\
	& \le \underbrace{\widetilde{\mu}_{I_t}^{\text{\red},h_{I_t,t}}(t) - \mu_{I_t}(N_{I_t,t})}_{\text{(a)}}+ \underbrace{2 \beta^{\text{\red},h_{I_t,t}}_{I_t}(t,\delta_t)}_{\text{(b)}},\label{:pp:-2002}
\end{align}	
where line~\eqref{:pp:-2001} follows from the definition of $B_{i}(t)$, and line~\eqref{:pp:-2002} derives from the fact that we are under the good event $\mathcal{E}_t$. We now decompose over the arms and consider one term at a time. We start with (a):
\begin{align}
&\sum_{t=1}^T \min \left\{1, \widetilde{\mu}_{I_t}^{\text{\red},h_{I_t,t}}(t) - \mu_{I_t}(N_{I_t,t}) \right\} \le\nonumber\\& \le 2K + \sum_{i \in [K]} \sum_{j=3}^{N_{i,T}}\min \left\{1, \widetilde{\mu}_{i}^{\text{\red},h_{i,t_{i,j}}}(t_{i,j}) - \mu_{i}(j) \right\} \\
& \le 2K \hspace{-0.13 cm}+ \hspace{-0.21 cm}\sum_{i \in [K]} \hspace{-0.1 cm} \sum_{j=3}^{N_{i,T}}\min \left\{1, \frac{1}{2} (2t_{i,j}\hspace{-0.15 cm} - 2(j-1)+ h_{i,t_{i,j}} \hspace{-0.16 cm}- 1)\gamma_{i}((j-1)-2h_{i,t_{i,j}}\hspace{-0.15 cm}+\hspace{-0.1 cm}1) \right\}  \label{:pr:001}\\
& \le 2K + \sum_{i \in [K]} \sum_{j=3}^{N_{i,T}}\min \left\{1, T \gamma_{i}(j- 2\lfloor \epsilon (j-1) \rfloor) \right\}\label{:pr:002}\\
& \le 2K + \sum_{i \in [K]} \sum_{j=3}^{N_{i,T}}\min \left\{1, T \gamma_{i}(\lfloor (1-2\epsilon) j \rfloor ) \right\}\label{:pr:003}\\
& \le 2K + T^q \sum_{i \in [K]} \sum_{j=3}^{N_{i,T}} \gamma_{i}(\lfloor (1-2\epsilon) j \rfloor )^q \label{:pr:0035}\\
& \le 2K + T^q \left\lceil \frac{1}{1-2\epsilon}  \right\rceil  \sum_{i \in [K]} \sum_{j=\lfloor 3(1-2\epsilon)\rfloor}^{\lfloor (1-2\epsilon) N_{i,T} \rfloor} \gamma_{i}( j ) \label{:pr:004}\\
& \le 2K + K T^q \left\lceil \frac{1}{1-2\epsilon}  \right\rceil  \Upsilon_{\bm{\mu}}\left( \left\lceil (1-2\epsilon) \frac{T}{K} \right\rceil,q\right),\label{:pr:007}
\end{align}
where line~\eqref{:pr:001} follows from Lemma~\ref{lemma:firstEstimator}, line~\eqref{:pr:002} is obtained by bounding $2t_{i,j} - 2(j-1)+ h_{i,t_{i,j}} - 1 \le 2T$ and exploiting the definition of $h_{i,t} = \lfloor \epsilon N_{i,t-1} \rfloor$, line~\eqref{:pr:003} follows from the observation $j- 2\lfloor \epsilon (j-1) \rfloor  \ge j- 2\epsilon (j-1)  \ge \lfloor (1-2\epsilon) j \rfloor$, line~\eqref{:pr:0035} is obtained from the already exploited inequality $\min\{1,x\} \le \min\{1,x\}^q \le x^q$ for $q \in [0,1]$, line~\eqref{:pr:004} is an application of Lemma~\ref{lemma:sumWithFloor}, {and line~\eqref{:pr:007} follows from Lemma~\ref{lemma:boundUpsilonMax} recalling that $\sum_{i\in [K]} \lfloor  (1-2\epsilon) N_{i,T} \rfloor \le (1-2\epsilon) T$.}

Let us now move to the concentration term (b). We decompose over the arms as well, taking care of the pulls in which $h_{i,j} = 0$, that are at most $1+\left\lceil \frac{1}{\epsilon} \right\rceil$:
\begin{align}
	&\sum_{t=1}^T \min\left\{1, 2 \beta^{\text{\red},h_{I_t,t}}_{I_t}(t,\delta_t)\right\} \le\nonumber\\& \le K \hspace{-0.1 cm}+\hspace{-0.1 cm} K\left\lceil \frac{1}{\epsilon} \right\rceil \hspace{-0.12 cm}+ \hspace{-0.21 cm}\sum_{i \in [K]} \sum_{j=\left\lceil \frac{1}{\epsilon} \right\rceil + 1}^{N_{i,T}}\hspace{-0.4 cm}\min \left\{1, 2\sigma (t_{i,t}- (j-1) + h_{i,t_{i,t}} \hspace{-0.14 cm}- 1) \sqrt{\frac{10 \log (t^\alpha)}{h_{i,t_{i,t}}^3}} \right\} \notag.\\
	& = K + K\left\lceil \frac{1}{\epsilon} \right\rceil + \sum_{i \in [K]} \sum_{j=\left\lceil \frac{1}{\epsilon} \right\rceil + 1}^{N_{i,T}}\hspace{-0.4 cm}\min \left\{1, 2\sigma T \sqrt{\frac{10 \alpha \log (T)}{ \lfloor \epsilon (j-1) \rfloor^3}} \right\},\label{:pr:008}
%
%
%
\end{align}
where line~\eqref{:pr:008} follows from bounding $t^\alpha \le T^\alpha$ and from the definition of $h_{i,t} = \lfloor \epsilon N_{i,t-1} \rfloor$. To bound the summation, we compute the minimum integer value $j^* $ (that turns out to be independent of $i$) of $j$ such that the minimum is attained by its second argument:
\begin{align*}
	2\sigma T \sqrt{\frac{10 \alpha \log (T)}{ \lfloor \epsilon (j-1) \rfloor^3}}  \le 1 & \implies \lfloor \epsilon (j-1) \rfloor \ge (2\sigma T)^{\frac{2}{3}} \left( 10 \alpha \log T \right)^{\frac{1}{3}} \\
	& \implies j^* = \left\lceil \frac{1+\epsilon + (2\sigma T)^{\frac{2}{3}} \left( 10 \alpha \log T \right)^{\frac{1}{3}}}{\epsilon} \right\rceil.
\end{align*}
Thus, we have:
\begin{align}
 & K +  K\left\lceil \frac{1}{\epsilon} \right\rceil + \sum_{i \in [K]} \sum_{j=\left\lceil \frac{1}{\epsilon} \right\rceil + 1}^{N_{i,T}} \min \left\{1, 2\sigma T \sqrt{\frac{10 \alpha \log (T)}{ \lfloor \epsilon (j-1) \rfloor^3}} \right\}\le \nonumber\\ &\le  K + K\left\lceil \frac{1}{\epsilon} \right\rceil + \sum_{i \in [K]} \left(\sum_{j=\lceil \frac{1}{\epsilon}\rceil+1}^{j^*} 1 + \sum_{j=j^*+1}^{N_{i,T}} 2\sigma T \sqrt{\frac{10 \alpha \log (T)}{ \lfloor \epsilon (j-1) \rfloor^3}} \right) \label{:pr:991}\\
& \le K + K\left\lceil \frac{1}{\epsilon} \right\rceil + K\left(j^*-1-\left\lceil \frac{1}{\epsilon} \right\rceil +1 \right) + \nonumber\\& \hspace{3 cm}+2 K \sigma T \sqrt{{10 \alpha \log (T)}} \int_{x=j^*}^{+\infty}  \frac{1}{(\epsilon (x-1) - 1)^\frac{3}{2}} \de x \label{:pr:992} \\
& = K + Kj^* + \frac{4 K \sigma T \sqrt{{10 \alpha \log (T)}}}{ \epsilon \left( \epsilon (j^*-1) - 1 \right)^{\frac{1}{2}}} \notag \\
& = K \left(3 + \frac{1}{\epsilon} \right)+ \frac{3K}{\epsilon}(2\sigma T)^{\frac{2}{3}} \left( 10 \alpha \log T \right)^{\frac{1}{3}}\label{:pr:993},
\end{align}
where line~\eqref{:pr:991} is obtained by splitting the summation based on the value of $j^*$, line~\eqref{:pr:992} comes from bounding the summation with the integral (Lemma~\ref{lemma:tecSum}), and line~\eqref{:pr:993} follows from substituting the value of $j^*$ and simple algebraic manipulations.
Putting all together, we obtain:
\begin{align*}
R_{\bm{\mu}}(\text{\algrested},T)  \le 1+ \frac{2K}{\alpha-2} &+ 5K + \frac{K}{\epsilon}+  \frac{3K}{\epsilon}(2\sigma T)^{\frac{2}{3}} \left( 10 \alpha \log T \right)^{\frac{1}{3}} +\\ &+K T^q \left\lceil \frac{1}{1-2\epsilon}  \right\rceil  \Upsilon_{\bm{\mu}}\left( \left\lceil (1-2\epsilon) \frac{T}{K} \right\rceil,q\right)\end{align*}
\begin{align*}
\hspace{1.7 cm} = \BigO \Bigg(  K T^q \left\lceil \frac{1}{1-2\epsilon}  \right\rceil  &\Upsilon_{\bm{\mu}}\left( \left\lceil (1-2\epsilon) \frac{T}{K} \right\rceil,q\right) +\\& +  \frac{K}{\epsilon}(\sigma T)^{\frac{2}{3}} \left(\alpha \log T \right)^{\frac{1}{3}} \Bigg) .
%
\end{align*}
\end{proof}

\clearpage

\clearpage
\section{Bounding the Cumulative Increment}\label{apx:example}
Let us consider the case in which $\gamma_i(l) \le l^{-c}$ for all $i \in [K]$ and $l \in[T]$. We bound the cumulative increment with the corresponding integral using Lemma~\ref{lemma:tecSum}, depending on the value of $cq$:
\begin{align*}
	\Upsilon_{\bm{\mu}}\left(\left\lceil \frac{T}{K} \right\rceil , q\right) = \sum_{l=1}^{\left\lceil \frac{T}{K} \right\rceil } {\max_{i \in [K]}}\gamma_i(l)^q &\le 1 + \int_{x=1}^{\frac{T}{K}} x^{-cq} \de x \nonumber\\ &\le 1+ \begin{cases}
		\left(\frac{T}{K}\right)^{1-cq} \frac{1}{1-cq} & \text{if } cq < 1 \\
		\log \frac{T}{K} & \text{if } cq = 1\\
		\frac{1}{cq-1} & \text{if } cq > 1
	\end{cases}.
\end{align*}
Thus, depending on the value of $c$, there will be different optimal values for $q$ in the rested and restless cases that optimize the regret upper bound. From Theorem~\ref{thr:theRegretDetRested}, we have:
\begin{align*}
	R_{\bm{\mu}} & \le 2K + T^q K  \Upsilon_{\bm{\mu}}\left(\left\lceil \frac{T}{K} \right\rceil,q \right) \le 2K + K T^q + K \begin{cases}
		\frac{T^{1-cq+q}}{K^{1-qc}(1-cq)} & \text{if } cq < 1 \\
		T^q \log  \frac{T}{K} & \text{if } cq = 1\\
		\frac{T^q}{cq-1} & \text{if } cq > 1
	\end{cases}\\
	&  \le \BigO \left( K \begin{cases}
		\frac{T^{1-cq+q}}{K^{1-qc}(1-cq)} & \text{if } cq < 1 \\
		T^q \log \frac{T}{K} & \text{if } cq = 1\\
		\frac{T^q}{\min\{1,cq-1\}} & \text{if } cq > 1
	\end{cases} \right) \qquad \forall q \in [0,1],
\end{align*}
where we have highlighted the dominant term. For the case $c \in (0,1)$ we consider the first case only and minimize over $q$:
\begin{align*}
R_{\bm{\mu}} \le \BigO \left(K \min_{q \in [0,1]}\frac{T^{1-cq+q}}{K^{1-qc}(1-cq)} \right) =  \BigO \left(T \right).
\end{align*}
For the case $c=1$, we still obtain $R_{\bm{\mu}} \le\BigO \left(T \right)$. Instead, for $c \in (1,+\infty)$, we have the three cases:
\begin{align*}
	R_{\bm{\mu}} \le\BigO \left( K \min \begin{cases}
		K \min_{q \in [0,1/c)}\frac{T^{1-cq+q}}{K^{1-qc}(1-cq)}  \\
		T^\frac{1}{c} \log \frac{T}{K} \\
		\min_{q \in (1/c,1]} \frac{T^q}{\min\{1,cq-1\}} 
	\end{cases}\right) = \BigO\left( K T^\frac{1}{c} \log \frac{T}{K}\right).
\end{align*}

Let us consider the case in which $\gamma_i(l) \le e^{-cl}$ for all $i \in [K]$ and $l \in[T]$. We bound the cumulative increment with the corresponding integral using Lemma~\ref{lemma:tecSum}, depending on the value of $cq$:
\begin{align*}
	\Upsilon_{\bm{\mu}}\left(\left\lceil \frac{T}{K} \right\rceil , q\right) = \sum_{l=1}^{\left\lceil \frac{T}{K} \right\rceil } {\max_{i \in [K]}}\gamma_i(l)^q \le 1 + \int_{x=1}^{\frac{T}{K}} e^{-cqx} \de x \le 1+ \frac{e^{-cq}}{cq}.
\end{align*}
By plugging the expression into the regret bound of Theorem~\ref{thr:theRegretDetRested}, we have:
\begin{align*}
	R_{\bm{\mu}} & \le 2K + T^q K  \Upsilon_{\bm{\mu}}\left(\left\lceil \frac{T}{K} \right\rceil,q \right) \le 2K + K T^q + KT^q\frac{e^{-cq}}{cq}.
\end{align*}
By choosing $q=1/\log(T)$, we obtain:
\begin{align*}
	R_{\bm{\mu}} & \le 2K + K T^\frac{1}{\log (T) } + \frac{K}{c} e^{1 - \frac{c}{\log
 (T)}}\log T \le \mathcal{O}\left( \frac{K}{c} \log T \right),
\end{align*}
having observed that $T^{\frac{1}{\log T}} =e $ and $e^{1 - \frac{c}{\log(T)} } \le e$.

\newpage
\section{Technical Lemmas}

\begin{lemma}\label{lemma:sumWithFloor}
	Let $M \ge 3$, and let $f: \Nat \rightarrow \Reals$, and $\beta \in (0, 1)$. Then it holds that:
	\begin{align*}
		\sum_{j=3}^{M} f(\lfloor \beta j \rfloor) \le \left\lceil \frac{1}{\beta} \right\rceil \sum_{l=\lfloor 3 \beta \rfloor}^{\lfloor \beta M \rfloor} f(l).
	\end{align*}
\end{lemma}

\begin{proof}
	We simply observe that the minimum value of $\lfloor \beta j \rfloor$ is $\lfloor 3\beta  \rfloor$ and its maximum value is $\lfloor \beta M \rfloor$. Each element $\lfloor \beta j \rfloor$ changes value at least one time every $\left\lceil \frac{1}{\beta} \right\rceil$ times.
\end{proof}

{
\begin{lemma}\label{lemma:boundUpsilonMax}
	Under Assumption~\ref{ass:decrDeriv}, it holds that:
	\begin{align*}
		\max_{\substack{(N_{i,T})_{i \in [K]} \\ N_{i,T} \ge 0, \sum_{i \in [K]}N_{i,T} = T}}   \;\;\sum_{i \in [K]} \sum_{l=1}^{N_{i,T}-1} \gamma_i(l)^q \le K \Upsilon_{\bm{\mu}}\left( \left\lceil \frac{T}{K} \right\rceil, q \right).
	\end{align*}
\end{lemma}

\begin{proof}
First of all, by definition of cumulative increment, we have that for every $(N_{i,T})_{i \in [K]}$:
\begin{align*}
\sum_{i \in [K]} \sum_{l=1}^{N_{i,T}-1} \gamma_i(l)^q \le \sum_{i\in [K]} \Upsilon_{\bm{\mu}}(N_{i,T},q).
\end{align*}
We now claim that there exists an optimal assignment of $N_{i,T}^*$ are such that $|N_{i,T}^*-N_{i',T}^*| \le 1$ for all $i,i' \in [K]$. By contradiction, suppose that the only optimal assignments are such that there exists a pair $i_1,i_2 \in [K]$ such that $\Delta \coloneqq N_{i_2,T}^* - N_{i_1,T}^* > 1$. In such a case, we have:
	\begin{align*}
		 &\Upsilon_{\bm{\mu}}\left( N_{i_1,T}^*, q \right) + \Upsilon_{\bm{\mu}}\left( N_{i_2,T}^*, q \right)  = 2\Upsilon_{\bm{\mu}}\left( N_{i_1,T}^*, q \right) + \sum_{j = 1}^{\Delta} \max_{i \in [K]}\gamma_{i}(N_{i_1,T}^*+l-1)^q  \\
		 & \le 2\Upsilon_{\bm{\mu}}\left( N_{i_1,T}^*, q \right) +  \sum_{j = 0}^{\lceil \Delta / 2 \rceil}  \max_{i \in [K]} \gamma_{i}(N_{i_1,T}^*+l-1)^q  + \sum_{j = 1}^{\lfloor \Delta / 2 \rfloor}  \max_{i \in [K]} \gamma_{i}(N_{i_1,T}^*+l-1)^q \\
		 & = \Upsilon_{\bm{\mu}}\left( N_{i_1,T}^* + {\lceil \Delta / 2 \rceil}, q \right) + \Upsilon_{\bm{\mu}}\left( N_{i_1,T}^* + \lfloor \Delta / 2 \rfloor, q \right),
	\end{align*}
	where the inequality follows from Assumption~\ref{ass:decrDeriv}. Thus, by redefining $\widetilde{N}_{i_1,T}^* \coloneqq N_{i_1,T}^* + {\lfloor \Delta / 2 \rfloor}$ and $\widetilde{N}_{i_2,T}^* \coloneqq N_{i_1,T}^* + {\lceil \Delta / 2 \rceil}$, we have that $\widetilde{N}_{i_1,T}^* + \widetilde{N}_{i_2,T}^* = {N}_{i_1,T}^* + {N}_{i_2,T}^*$ and $|\widetilde{N}_{i_1,T}^* -\widetilde{N}_{i_2,T}^*| \le 1$. Thus, we have found a better solution to the optimization problem, contradicting the hypothesis. Since the optimal assignment fulfills $|N_{i,T}^*-N_{i',T}^*| \le 1$, it must be that $N_{i,T}^* \le \left\lceil \frac{T}{K} \right\rceil$ for all $i \in [K]$.
\end{proof}
}

\begin{restatable}[]{lemma}{boundDeriv}\label{lemma:boundDeriv}
	Under Assumptions~\ref{ass:incr} and~\ref{ass:decrDeriv}, for every $i\in [K]$, $k,k' \in \Nat$ with $k'<k$, for both rested and restless bandits, it holds that:
	\begin{align*}
		\gamma_i(k) \le \frac{\mu_i(k) - \mu_i(k')}{k-k'}.
	\end{align*}
\end{restatable}

\begin{proof}
	Using Assumption~\ref{ass:decrDeriv}, we have:
	\begin{align*}
		\gamma_i(k) = \frac{1}{k-k'} \sum_{l=k'}^{k-1} \gamma_i(k) \le \frac{1}{k-k'} \sum_{l=k'}^{k-1} \gamma_i(l) &= \frac{1}{k-k'} \sum_{l=k'}^{k-1} \big( \mu_i(l+1) - \mu_i(l) \big) \nonumber\\ &= \frac{\mu_i(k) - \mu_i(k')}{k-k'},
	\end{align*}
	where the first inequality comes from the concavity of the reward function, and the second equality from the definition of increment.
\end{proof}

\begin{lemma}\label{lemma:tecSum} 
Let $a,b \in \Nat$ and let $f: [a,b] \rightarrow \Reals$. If $f$ is monotonically non-decreasing function, then:
\begin{align*}
	\sum_{n=a}^b f(n) \le \int_{x = a}^b f(x) \de x + f(b) \le \int_{x = a}^{b+1} f(x) .
\end{align*}
If $f$ is monotonically non-increasing, then:
\begin{align*}
	\sum_{n=a}^b f(n) \le f(a) + \int_{x = a}^{b} f(x) \de x \le \int_{x = a-1}^{b} f(x) \de x.
\end{align*}
\end{lemma}

\begin{proof}
	Let us consider the intervals $I_i = [x_{i-1},x_{i}]$ with $x_{0} = a$ and $x_{i} = x_{i-1} + 1$ for $i \in [b-a]$. If $f$ is monotonically non-decreasing, we have that for all $i \in [b-a]$ and $x \in I_i$ it holds that $f(x) \ge f(x_{i-1})$ and consequently $\int_{I_i} f(x) \de x \ge f(x_{i-1}) \mathrm{vol}(I_{i}) = f(x_{i-1})$. Thus:
	\begin{align*}
		\sum_{n=a}^b f(n) = \sum_{i=1}^{b-a} f(x_{i-1}) + f(b) \le \sum_{i=1}^{b-a} \int_{I_i} f(x) \de x + f(b) = \int_{x=a}^b f(x) \de x + f(b).
	\end{align*}
	Recalling that $f(b) \le \int_{x=b}^{b+1} f(x) \de x$, we get the second inequality.
	Conversely, if $f$ is monotonically non-increasing, then for all $i \in [b-a]$ and $x \in I_i$, it holds that $f(x) \ge f(x_{i})$ and consequently $\int_{I_{i}} f(x) \de x \ge f(x_{i})$. Thus:
	\begin{align*}
		\sum_{n=a}^b f(n) = f(a) + \sum_{i=1}^{b-a} f(x_{i})  \le f(a) + \sum_{i=1}^{b-a} \int_{I_i} f(x) \de x = f(a) + \int_{x=a}^b f(x) \de x .
	\end{align*}
	Recalling that $f(a) \le \int_{x=a-1}^{a} f(x) \de x$, we get the second inequality.
\end{proof}

\begin{thr}[H\"oeffding-Azuma’s inequality for weighted martingales] \label{lemma:hoeffding}
Let $\mathcal{F}_1 \subset \dots \subset \mathcal{F}_n$ be a filtration and $X_1, \dots, X_n$ be real random variables such that $X_t$ is $\mathcal{F}_t$-measurable, $\E[X_t|\mathcal{F}_{t-1}]=0$ (\ie a martingale difference sequence), and $\E[\exp(\lambda X_t)|\mathcal{F}_{t-1}]\le \exp\left( \frac{\lambda^2 \sigma^2}{2} \right)$ for any $\lambda >0$ (\ie $\sigma^2$-subgaussian). Let $\alpha_1, \dots, \alpha_n$ be non-negative real numbers. Then,  for every $\kappa \ge 0$ it holds that:
\begin{align*}
\Pr \left( \left| \sum_{t=1}^n \alpha_t X_t \right| >\kappa\right) \le 2 \exp \left( -\frac{\kappa^2}{2 \sigma^2 \sum_{t=1}^n \alpha_i^2}  \right).
\end{align*}
%
%
%
%
\end{thr}
\begin{proof}
	It is a straightforward extension of Azuma-H\"oeffding inequality for subgaussian random variables. We apply the Chernoff's method for some $s > 0$:
	\begin{align*}
	\Pr \left(  \sum_{t=1}^n \alpha_t X_t >\kappa\right) & = \Pr \left(  e^{ s \sum_{t=1}^n \alpha_t X_t}> e^{s\kappa}\right) \le \frac{\E\left[  e^{ s \sum_{t=1}^n \alpha_t X_t} \right]}{ e^{s\kappa}},
	\end{align*}
	where the last inequality follows from the application of Markov's inequality. We use the martingale property to deal with the expectation. By the law of total expectation, we have:
	\begin{align*}
	\E\left[  e^{ s \sum_{t=1}^n \alpha_t X_t} \right] = \E\left[  e^{ s \sum_{t=1}^{n-1} \alpha_t X_t} \E \left[e^{ s \alpha_n X_n}  \rvert \mathcal{F}_{t-1} \right] \right].
	\end{align*}
	Using now the subgaussian property, we have:
	\begin{align*}
	\E \left[e^{ s \alpha_n X_n}  \rvert \mathcal{F}_{t-1} \right]  \le \exp \left( \frac{s^2\alpha_n^2 \sigma^2}{2} \right).
	\end{align*}
	An inductive argument, leads to:
	\begin{align*}
	\E\left[  e^{ s \sum_{t=1}^n \alpha_t X_t} \right] \le  \exp \left( \frac{s^2 \sigma^2}{2} \sum_{t=1}^n \alpha_n^2 \right).
	\end{align*}
	Thus, minimizing \wrt $s > 0$, we have:
	\begin{align*}
		\Pr \left(  \sum_{t=1}^n \alpha_t X_t >\kappa\right)  \le\min_{s \ge 0} \, \exp \left( \frac{s^2 \sigma^2}{2} \sum_{t=1}^n \alpha_n^2 - s\kappa \right) =   \exp\left( -\frac{\kappa^2}{2  \sigma^2 \sum_{t=1}^n \alpha_n^2}\right),
	\end{align*}
	being the minimum attained by $s = \frac{\kappa}{\sigma^2 \sum_{t=1}^n \alpha_n^2}$. The reverse inequality can be derived analogously. A union bound completes the proof.
\end{proof}

\begin{lemma}\label{lemma:lemmaBoundsGamma}
	Let $\Upsilon_{\bm{\mu}}(T,q)$ be as defined in Equation~\eqref{eq:magicQuantity} for some $q \in [0,1]$. Then, for all $i \in [K]$ and $l \in \Nat$ the following statements hold:
	\begin{itemize}
		\item if $\gamma_i(l) \le b e^{-cl}$, then $\Upsilon_{\bm{\mu}}(T,q) \le  \BigO\left( b^{q} \frac{e^{-cq}}{cq} \right)$;
		\item if $\gamma_i(l) \le b l^{-c}$  with $cq > 1$, then $\Upsilon_{\bm{\mu}}(T,q) \le \BigO\left( \frac{b^q}{cq - 1} \right)$;
		\item if $\gamma_i(l) \le b l^{-c}$  with $cq = 1$, then $\Upsilon_{\bm{\mu}}(T,q) \le \BigO\left(  b^q \log T\right)$;
		\item if $\gamma_i(l) \le b l^{-c}$  with $cq < 1$, then $\Upsilon_{\bm{\mu}}(T,q) \le \BigO\left( b^q\frac{T^{1-cq}}{1-cq} \right)$.
	\end{itemize}
\end{lemma}

\begin{proof}
The proofs of all the statements are obtained by bounding the summation defining $\Upsilon_{\bm{\mu}}(T,q)$ with the corresponding integrals, as in Lemma~\ref{lemma:tecSum}. Let us start with $\gamma_i(l) \le b e^{-cl}$:
\begin{align*}
	\Upsilon_{\bm{\mu}}(T,q) = \sum_{l=1}^T \gamma_i(l)^q \le  b^qe^{-cq} + \int_{x=1}^{T} b^qe^{-cqx} \de x &\le b^qe^{-cq} + \frac{b^q}{cq} e^{-cq} \\&= \BigO\left( b^{q} \frac{e^{-cq}}{cq} \right).
\end{align*}
We now move to $\gamma_i(l) \le b l^{-c}$. If $cq < 1$, we have:
\begin{align*}
\Upsilon_{\bm{\mu}}(T,q) = \sum_{l=1}^T \gamma_i(l)^q \le b^q + \int_{x=1}^{T} b^q x^{-cq} \de x = b^q + \frac{b^q}{cq - 1} = \BigO\left( \frac{b^q}{cq - 1} \right).
\end{align*}
For $cq=1$, we obtain:
\begin{align*}
\Upsilon_{\bm{\mu}}(T,q) = \sum_{l=1}^T \gamma_i(l)^q \le b^q + \int_{x=1}^{T} \frac{b^q}{x} \de x = b^q + b^q \log T = \BigO\left(  b^q \log T\right).
\end{align*}
Finally, for $cq < 1$, we have:
\begin{align*}
	\Upsilon_{\bm{\mu}}(T,q) = \sum_{l=1}^T \gamma_i(l)^q \le b^q + \int_{x=1}^{T} b^q x^{-cq} \de x = b^q + b^q\frac{T^{1-cq}}{1-cq} = \BigO\left( b^q\frac{T^{1-cq}}{1-cq} \right).
\end{align*}
The results of Table~\ref{tab:rates} are obtained by setting $b=1$.
\end{proof}

\lowerBoundWithDeltaBar*

\begin{proof}
    Let $T_i(T)$ be the r.v. representing the number of pulls at the time horizon $T$ for the $i-th$ arm, then the regret for an arbitrary experiment can be written as:
    \begin{align}
        r_{\bm{\mu}}(\pi,T)=\sum_{n=T_1(T)+1}^{T} \mu_1(n)-\sum_{i\ge 2}\sum_{n=1}^{T_i(T)}\mu_i(n)
    \end{align}
    By definition of the rising rested problem the following inequalities will hold true for all possible values of $T_1(T)$:
    \begin{align}
        \sum_{n=T_1(T)+1}^{T} \mu_1(n) \ge \sum_{n=T_1(T)+1}^{T} \overline{\mu}_1(T),
    \end{align}
    being $ \overline{\mu}_1(T)=\frac{1}{T}\sum_{n=1}^T\mu_1(n)$, conversely for the same reason it will also hold true for all the possible values of $T_i(T)$ that:
    \begin{align}
        \sum_{n=1}^{T_i(T)} \mu_i(n) \leq \sum_{n=1}^{T_i(T)} \overline{\mu}_i(T),
    \end{align}
    being $ \overline{\mu}_i(T)=\frac{1}{T}\sum_{n=1}^T\mu_i(n)$, so that the regret for a single instantation can be lower bounded as:
    \begin{align} 
        r_{\bm{\mu}}(\pi,T)&=\sum_{n=T_1(T)+1}^{T} \mu_1(n)-\sum_{i\ge 2}\sum_{n=1}^{T_i(T)}\mu_i(n)\\
        &\ge \sum_{n=T_1(T)+1}^{T} \overline{\mu}_1(T)-\sum_{i\ge 2}\sum_{n=1}^{T_i(T)}\overline{\mu}_i(T)\\
        &= \sum_{i\ge 2}T_i(T)\overline{\Delta}_i
    \end{align}
    so that applying the expectation value wrt the policy we do obtain:
    \begin{align}
     \mathbb{E}[r_{\bm{\mu}}(\pi,T)]=R_{\bm{\mu}}(\pi,T)=\sum_{i\ge 2} \overline{\Delta}_i\mathbb{E}_{\bm{\mu}}^{\pi}[T_i(T)]
    \end{align}
\end{proof} 

\begin{lemma}[Bretagnolle-Huber Inequality]\label{lemma:bret}
    Let $P$ and $Q$ be two probability distributions on the same measurable space $(\Omega, \mathcal{F})$, and let $A \in \mathcal{F}$ be an arbitrary event. Then:
    \begin{align}
        P(A)+Q\left(A^c\right) \geq \frac{1}{2} \exp (-\mathrm{D}(P, Q)),
    \end{align}
where $A^c=\Omega \backslash A$ is the complement of $A$.
\end{lemma}

\begin{lemma}[Radon-Nykodin Theorem, \cite{billingsley1996probandmeasure}, Theorem 32.2]\label{lemma:radnyd}
Let $\mu$ and $\nu$ be two measures on $(\Omega,\mathcal{F})$, if $\mu$ and $\nu$ are two $\sigma$-measure such that $\nu \ll \mu$, then there exist a non-negative $f$, a density, such that:
	\begin{align}
		\nu(A)=\int_{A} f d\mu
	\end{align}
	for all $A \in \mathcal{F}$.
\end{lemma}

\section{Efficient Update}\label{apx:efficient}
Under the assumption that the window size depends on the number of pulls only and that $0 \le h({n+1}) - h({n}) \le 1$, we can employ the following efficient $\mathcal{O}(1)$ update for \algrested. Denoting with $n$ the number of pulls of arm $i$, we update the estimator at every time step $t \in [T]$ as:
\begin{align*}
	\widehat{\mu}_{i}^{h(n)}(t) = \frac{1}{h(n)} \left( a_n + \frac{t(a_n - b_n)}{h(n)}  - \frac{c_n - d_n}{h(n)} \right),
\end{align*}
where the following sequences are updated only when the arm is pulled:
\begin{align*}
	& a_{n} = \begin{cases}
					a_{n-1} + r_i(n)-r_i(n-h(n)) & \text{if } h({n}) = h({n-1}) \\
					a_{n-1} + r_i(n) & \text{otherwise}
				\end{cases}, \\
	&  b_{n} = \begin{cases}
					b_{n-1} +r_i(n-h(n)) - r_i(n-2h(n)) & \text{if } h({n}) = h({n-1}) \\
					b_{n-1} + r_i(n-2 h(n)+1) & \text{otherwise}
				\end{cases}, \\
	&  c_{n} = \begin{cases}
					c_{n-1} +nr_i(n) - (n-h(n))r_i(n-h(n)) & \text{if } h({n}) = h({n-1}) \\
					c_{n-1} + nr_i(n) & \text{otherwise}
				\end{cases}, \\
	&  d_{n} = \begin{cases}
					d_{n-1} +n r_i(n-h(n)) -(n-h(n)) r_i(n-2h(n))& \text{if } h({n}) = h({n-1}) \\
					d_{n-1} + (n-h(n))r_i(n-2h(n)+1) + b_n & \text{otherwise}
				\end{cases},
\end{align*}

where we have abbreviated $r_i(n) \coloneqq R_{t_{i,n}}$.

\section{Experimental Setting and Additional Results}\label{apx:experiments}

\subsection{Parameter Setting}
The choices of the parameters of the algorithms we compared \texttt{R-less/ed-UCB} with are the following:
\begin{itemize}
	\item \texttt{Rexp3}: $V_T = K$ since in our experiments we consider the reward of each arm to evolve from $0$ to $1$, thus the maximum global variation possible is equal the number of arms of the bandit; $\gamma = \min \left\{ 1, \sqrt{\frac{K\log{K}}{(e-1)\Delta_T}} \right\}$, $\Delta_T = \lceil (K\log{K})^{1/3} (T/V_T)^{2/3} \rceil$ as recommended by~\citet{besbes2014stochastic};
	\item \texttt{KL-UCB}: $c = 3$ as required by the theoretical results on the regret provided by~\citet{garivier2011kl};
	\item \texttt{Ser4}: according to what suggested by~\citet{allesiardo2017non} we selected $\delta=1/T$, $\epsilon=\frac{1}{KT}$, and $\phi = \sqrt{\frac{N}{TK\log({KT})}}$;
	\item \texttt{SW-UCB}: as suggested by~\citet{garivier2011upper} we selected the sliding-window $\tau = 4\sqrt{T\log{T}}$ and the constant $\xi = 0.6$;
	\item \texttt{SW-KL-UCB} as suggested by~\citet{garivier2011upper} we selected the sliding-window $\tau = \sigma^{-4/5}$;
	\item \texttt{SW-TS}: as suggested by~\citet{trovo2020sliding} for the smoothly changing environment we set $\beta = 1/2$ and sliding-window $\tau = T^{1-\beta} = \sqrt{T}$.
\end{itemize}

\subsection{IMDB Experiment}\label{apx:imdb}
We created a bandit environment in which each of the classification algorithms is an arm of the bandit.
The interaction for each round $t \in T$ of the real-world experiment is composed by the following:
\begin{itemize}
	\item the agent decides to pull arm $I_t$;
	\item a random point $x_t$ of the IMDB dataset is selected and supplied to the classification algorithm associated to arm $I_t$;
	\item the \quotes{base} algorithm classifies the sample, \ie~it provides the prediction $\hat{y}_t \in \{0,1\}$ for the selected sample $x_t$;
	\item the environment generates the reward comparing the prediction $\hat{y}_t$ to the target class $y_t$ using the following function $R_t = 1 - |y_t - \hat{y}_t|$;
	\item the base algorithm is updated using $(x_t, y_t)$;
\end{itemize}
Since the base algorithms are trained only if their arm is selected, this is a problem which belongs to the rested scenario.

For the classification task we decided to employ:
\begin{itemize}
	\item $2$ Online Logistic Regression (LR) methods with different schemes used for the learning rate $\lambda_t$;
	\item $5$ Neural Networks (NNs) different in terms of shape and number of neurons
\end{itemize}
Specifically, we adopt a decreasing scheme for the learning rate of $\lambda_t = \frac{\beta}{t}$ (denoted with \texttt{LR}$(t)$ from now on) and a constant learning rate $\lambda_t = \beta$ (denoted as \texttt{LR} from now on).
Moreover, the NNs use as activation functions the rectified linear unit, \ie $relu(x) = \max(0,x)$, a constant learning rate $\alpha = 0.001$ and the \quotes{adam} stochastic gradient optimizer for fitting.
Two of the chosen nets have only one hidden layer, with $1$ and $2$ neurons, respectively, the third net has $2$ hidden layer, with $2$ neurons each, and two nets have $3$ layers with $2,2,2$ and $1,1,2$ neurons, respectively.
We refer to a specific NN denoting in curve brackets the cardinalities of the layers, \eg the one having $2$ layer with $2$ neurons each is denoted by \texttt{NN}$(2,2)$.

We analyzed their global performance on the IMDB dataset by averaging $1,000$ independent runs in which each strategy is sequentially fed with all the available $50,000$ samples.
The goal was to determine, at each step, the value of the expected reward $\mu_i(n)$. Figure~\ref{fig:imdb_learning_curves} provides the average learning curves of the selected algorithms.
As we expected, from a qualitative perspective, the average learning curves are increasing and concave, however, due to the limited number of simulations, Assumptions~\ref{ass:incr} and~\ref{ass:decrDeriv} are not globally satisfied.

We also perform an experiment using only \texttt{LR}$(t)$ and \texttt{LR} as arms.
Figure~\ref{fig:imdb_old} reports the result of a run of the MAB algorithms over the IMDB scenario.
The analogy between this result and the one of the $2$-arms synthetic rested bandit (Figure~\ref{fig:rested_2arms_regrets}) is clear, indeed \algrested outperforms the other baselines when the learning curves of the base algorithms at some point intersects one another. 

\begin{figure*}
	\vspace{.25cm}
	\begin{minipage}{.48\textwidth}
		\centering
		\subfloat{\scalebox{0.9}{\includegraphics{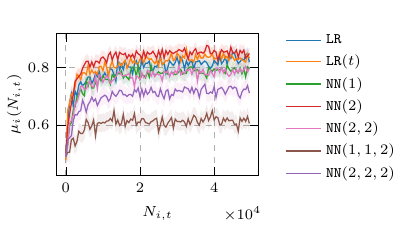} \label{fig:imdb_learning_curves}}}
		\captionof{figure}{Empirical learning curves of the classification algorithms (arms) of the IMDB experiment}
	\end{minipage}%
	\hfill
	\begin{minipage}{.48\textwidth}
		\centering
		\subfloat{\scalebox{0.9}{\includegraphics{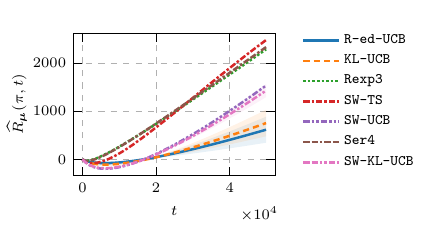} \label{fig:imdb_old}}}
		\captionof{figure}{Cumulative regret in a $2$-arms online model selection on IMDB dataset ($30$ runs, $95$\% c.i.).}
	\end{minipage}
	\end{figure*}

\subsection{Pulls of each arm}

Figure~\ref{fig:pulls} presents the average number of pulls for each arm for each one of the algorithm analysed in the synthetic experiments of Section~\ref{sec:experiments}.
Figure~\ref{fig:rested_15arms_pulls} shows that \algrested{} explored arms $13$ and $1$ more than the others, which are respectively the best and the second-best, and most likely needs a longer time horizon to select which one is the best among the twos.
Figure~\ref{fig:rested_2arms_pulls} highlights the fact that \algrested{} undoubtedly identified which arm is the best (arm $2$), while \texttt{KL-UCB}, \texttt{SW-UCB}, \texttt{SW-KL-UCB} do not identify the best arm. \texttt{Ser4}, \texttt{Rexp3} and \texttt{SW-TS} pulled the best arm slightly more than $50\%$ of the times, paying the already discussed initial learning phase.

\begin{figure}
	\centering
	\subfloat{\scalebox{0.75}{\includegraphics{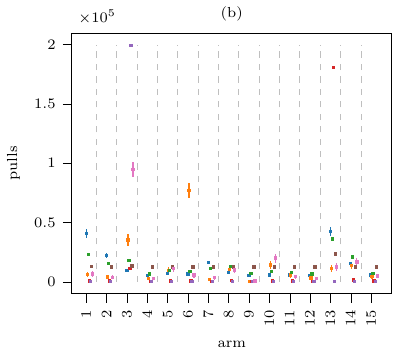} \label{fig:rested_15arms_pulls}}}
	\hfill
	\subfloat{\scalebox{0.75}{\includegraphics{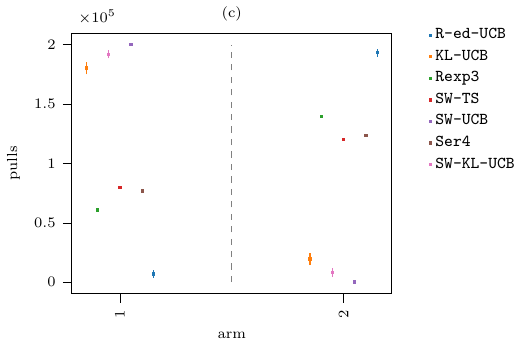} \label{fig:rested_2arms_pulls}}}
	\caption{Average number of pulls: (a)~$15$ arms \red, (b)~$2$ arms \red.} \label{fig:pulls}
\end{figure}
\end{document}